\documentclass[journal]{IEEEtran}
\usepackage[final]{changes}

\usepackage{cite}

\ifCLASSINFOpdf
\else
\fi
\usepackage{amsmath}
\usepackage{algorithm}
\usepackage{algpseudocode}
\ifCLASSOPTIONcompsoc
 \usepackage[caption=false,font=normalsize,labelfont=sf,textfont=sf]{subfig}
\else
 \usepackage[caption=false,font=footnotesize]{subfig}
\fi
\usepackage{graphicx}
\usepackage{amsfonts}
\usepackage{hyperref}
\hypersetup{
    colorlinks=true,
    linkcolor=blue,
    filecolor=magenta,      
    urlcolor=cyan,
    pdftitle={},
    pdfpagemode=FullScreen,
    }

\usepackage{xcolor}

\usepackage{amsmath}
\usepackage{amssymb}
\usepackage{amsfonts}
\usepackage{mathtools}
\usepackage{amsthm}
\usepackage{multicol}

\newtheorem{prop}{Proposition}

\theoremstyle{definition}
\newtheorem{definition}{Definition}
\newtheorem*{remark}{Remark}

\newcommand{\vc}[1]{{\mathbf #1}}
\newcommand{\vct}[1]{\pmb{#1}}
\newcommand{\ma}[1]{{\mathbf #1}}

\newcommand{\conjugate}{^{\star}}

\newcommand{\Tr}{\mathrm{Tr}}

\begin{document}
%
\title{Statistical Testing on Directed Graphs by\\ Surrogate Data Generation}
%
%
%

\author{Chun Hei Michael~Chan,~\IEEEmembership{Student Member,~IEEE}, %
        Alexandre~Cionca,~\IEEEmembership{Student Member,~IEEE}, 
     \\
    Dimitri~Van~De~Ville,~\IEEEmembership{Fellow,~IEEE}

\thanks{CHM. Chan, A. Cionca and D. Van De Ville are with the Neuro-X Institute, Ecole polytechnique fédérale de Lausanne, and the Department of Radiology and Medical Informatics, University of Geneva, Switzerland. E-mail: chunheimichael.chan@epfl.ch; dimitri.vandeville@epfl.ch}%
\thanks{This work was supported by the Swiss National Science Foundation, Sinergia project ``Precision mapping of electrical brain network dynamics with application to epilepsy'', Grant number 209470.}%
\thanks{Manuscript received XXX; revised XXX.}}

%
%

\markboth{Submitted to IEEE Transactions on Signal and Information Processing over Networks}%
{}
%



\maketitle

\begin{abstract}
In recent years, graph signal processing has emerged as a powerful framework at the intersection of signal processing and graph theory, providing tools for the analysis of signals defined on nodes while accounting for their relationships represented by edges. These tools have been successfully applied to various settings, including statistical hypothesis testing. In particular, non-parametric approaches based on surrogate generation have been proposed for signals on undirected graphs. However, they are yet to be extended to directed graphs. In this work, we first revisit the notion of stationary graph signals on directed graphs. Specifically, and through the eigendecomposition of the graph shift operator, we define directed graph wide-sense stationary signals. Then, we propose a new framework to generate surrogate graph signals that preserve covariance structure under stationarity assumptions. Null distributions of the test metric can then be constructed from these surrogates and serve as a reference for the empirical data. Finally, we provide guiding examples and an application on real data, in which we compare the performance of our framework with existing techniques for undirected graphs or based on naive permutation, demonstrating feasibility and superiority of the proposed approach.
\end{abstract}


%
\IEEEpeerreviewmaketitle

\section{Introduction}
%
%
%
%
\IEEEPARstart{G}{raph} signal processing (GSP) has emerged as a compelling methodology for managing data expressed on structured domains; e.g., in the context of social, biological, technological networks. GSP considers graphs and graph signals in a common framework to revisit classical signal processing operations \cite{shuman_emerging_2013, ortega_graph_2018, marques_signal_2020}. In the probabilistic setting, where graph signals are considered as realizations of random processes, GSP is essentially limited to undirected graphs, which motivates this work to provide a framework to enable statistical hypothesis testing on directed graphs. We emphasize that, in contrast to the extensive body of work on directionality inference for causal graph discovery~\cite{cui_topology_2024, sanchez_diffusion_2022, shaska_causal_2025, ajorlou2025build}, the present study addresses statistical hypothesis testing for assessing signals observed on a predefined graph. Several differences also exist in the type of graphs that each framework considers. 

In conventional signal processing, a random process in the time domain is said to be weakly stationary when the first two moments, mean and autocorrelation, do not vary with respect to time. Equivalently, its covariance matrix is diagonalized by the discrete Fourier transform (DFT) basis~\cite{brockwell_time_1991}. The latter concept has been used for defining graph wide-sense stationarity (GWSS) for undirected graphs~\cite{girault_stationary_2015, perraudin_stationary_2017, marques_stationary_2017}; i.e., a random process on the graph is GWSS if and only if its covariance matrix is diagonalized by the eigenvectors of the graph shift operator (GSO), commonly the adjacency matrix or the graph Laplacian, that can be considered as the graph Fourier transform (GFT). Stationarity of graph signals has been leveraged in a variety of applications, notably Wiener-based graph signal denoising, the forecasting of graph signal timeseries and graph structure learning~\cite{cheung_graph_2018,isufi_forecasting_2019, egilmez2018graph, buciulea_graph_2023}. For non-parametric statistical testing, under the stationarity assumption, we proposed a scheme to generate surrogate graph signals~\cite{pirondini_spectral_2016} based on sign randomization of the GFT coefficients. For a chosen test metric, the value from the empirical data can then be compared against the null distribution obtained from surrogates. 

Recently, the notion of stationarity has been extended to directed graphs~\cite{iraji_wss_2025} leading to directed GWSS (DGWSS). Here, the covariance matrix can be expressed as a left–right product of a block-diagonal matrix by the generalized eigenvector matrix of the GSO and its Hermitian transpose. However, this formulation relies on the Jordan Normal Form (JNF) of the GSO, which presents two major limitations: the JNF is numerically unstable and computationally inefficient. Yet, most graphs are strongly connected and thus come with a diagonalizable directed GSO~\cite{sevi_harmonic_2023}, and, if not, they can be made so by a minimal perturbation~\cite{seifert_digraph_2021}. This provides a complex-valued eigenbasis, but that is not necessarily orthogonal. This allows to obtain meaningful phase information that can be leveraged for operations such as the graph Hilbert transform~\cite{chan_hilbert_2025, venkitaraman_hilbert_2019}. Surrogate generation on directed graphs has not yet been addressed, except for the special case of complex-valued Hermitian graphs~\cite{belda_new_2019}.

In this work, we first revisit the notion of directed graph wide sense stationarity (DGWSS) through the eigendecomposition of the GSO and extend properties from GWSS to DGWSS. Next, we study white noise (WN) on the graph, a particular type of stationary signal with flat power spectral density, highlighting some striking differences between WN on undirected and directed graphs. We then present the framework for non-parametric testing, leveraging our proposed surrogate generation scheme. Throughout all experiments, we will compare our framework against surrogates generated on the undirected graph, and surrogates from naive permutations, demonstrating that directionality is meaningful and useful to be accounted for. We also show results for experimental data from temperature measurements on a sensor network. 

\added{The remainder of the paper is organized as follows. In Sect.~\ref{sect:gft}, we introduce the directed graph signal processing framework. Then, in Sect.~\ref{sect:statio}, we revisit the definition and properties of GWSS and its extension to directed graphs. We focus in Sect.~\ref{sect:white-noise} on white noise for directed graphs and discuss its differences with the undirected setting. Next, in Sect.~\ref{sect:surrogate}, we present the proposed non-parametric testing framework together with the surrogate generation scheme. Numerical experiments on synthetic and real-world datasets are presented in Sect.~\ref{sect:experimentals}, including temperature data on a sensor network. Finally, we conclude the paper and outline future research directions.}
\section{Notations}
Lower-case bold letters indicate vectors; e.g., $\vc x$ where $\vc x[n]$ is the $n$-th element. When the vector indicates a random variable $\vc x$, different realizations will be denoted with a subscript as $\vc x_k$, and surrogates generated from such a realization with a superscript as $\vc x_k^{(l)}$. Further on, upper-case bold letters are used for matrices; e.g., $\ma U$  with $\ma U[n,m]$ denoting the $(n,m)$-th entry of $\ma U$. The notations $(.)^*$, $(.)^T$, $(.)^H$, $(.)^{-1}$, $(.)^{-H}$ denote conjugate, transpose, Hermitian, inverse, and inverse Hermitian, respectively. We use $j$ to refer to the complex number such that $j^2=-1$. With $||.||_2$ and $||.||_F$, we refer to \replaced{$L^2$}{L2}-norm and Frobenius norm, respectively. We use \added{$\bar{(\cdot)}$ to refer to matrices related to the undirected graph.}

\section{Graph Fourier Transform}
\label{sect:gft}

\subsection{Undirected Graph}
Consider an undirected graph with $N$ nodes represented by its symmetric adjacency matrix $\bar{\ma A}$, which can be eigendecomposed into
\begin{equation}
  \bar{\ma A}=\bar{\ma U}\bar{\ma \Lambda}\bar{\ma U}^T,  
\end{equation}
where $\bar{\ma U}$ contains the real-valued eigenvectors $[\bar{\vc u}_k]_{k=1,\ldots,N}$ that form an orthonormal basis with associated eigenvalues on the diagonal of $\bar{\ma \Lambda}$. From this decomposition, the GFT of a graph signal $\vc x$ is defined as $\hat{\vc x}={\bar{\ma U}}^T {\vc x}$.

Applying the GFT brings the initial graph signal into the spectral domain, where spectral filtering can then be performed as the multiplication by a diagonal matrix $\bar{\ma G}\in\mathbb{R}^{N\times N}$, before being brought back into the signal domain:
\begin{equation}
   \vc y = \bar{\ma U}\bar{\ma G}\bar{\ma U}^T\vc x.
\end{equation}
\subsection{Directed Graph}
For a directed graph with $N$ nodes, characterized by the (non-symmetric) adjacency matrix $\ma A$, the eigendecomposition is not guaranteed and generally one can revert to the Jordan normal form (JNF)~\cite{sandryhaila_discrete_2013}. The matrix $\ma A$ then takes the following form, $\ma A=\ma V \ma J\ma V^{-1}$, with $\ma J$ a block diagonal matrix composed of $K$ Jordan blocks $(\ma J_k)_{k=1,\ldots,K}$ each associated with the eigenvalue $\lambda_k$ of the following form:
\begin{equation}
    \ma J = \begin{bmatrix}
    \ma J_1 & & \\
    & \ddots & \\
    & & \ma J_K
    \end{bmatrix}, \text{where} \ \ma J_k = 
    \begin{bmatrix}
    \lambda_k & 1 & & \\
    & \lambda_k & \ddots & \\
    & & \ddots & 1 \\
    & & & \lambda_k
    \end{bmatrix}.
\end{equation}
The columns of the matrix $\ma V$ are the generalized eigenvectors of $\ma A$. Filtering in spectral domain is defined as 
\begin{equation}
    \vc y = \ma V \ma G^{\ma J} \ma V^{-1}\vc x,
\end{equation}
where $\ma G^{\ma J}$ is a block diagonal matrix defined through a  complex-valued polynomial function $g$~\cite{sandryhaila_discrete_2013} as
\begin{equation}
    \ma G^{\ma J} = \begin{bmatrix}
    \ma G^{\ma J}_1 & & \\
    & \ddots & \\
    & & \ma G^{\ma J}_K
    \end{bmatrix}, \qquad \text{where} \  \ma G^{\ma J}_k = g(\ma J_k).
\end{equation}

While it might look natural to consider the JNF, it is well known that its computation is unfeasible even for graphs of moderate size ~\cite{stewart2001matrix}. In addition, for many graphs of practical interest, the eigendecomposition of $\ma A$ does exist and, if not, a minimal perturbation to $\ma A$ can be applied~\cite{seifert_digraph_2021} to allow for the diagonalization of $\ma A$. Hence, we consider the eigendecomposition as 
\begin{equation}
  \ma A=\ma U\ma \Lambda \ma U^{-1},
\end{equation}
where the eigenvalues $\lambda_k = \ma \Lambda[k,k]$ are either real-valued or complex conjugate pairs, with associated real-valued or complex conjugate eigenvectors, respectively. The GFT of a real-valued graph signal $\vc x$ is defined as $\hat{\vc x}=\ma U^{-1}{\vc x}$. For pairs of indices $(k_1, k_2)$ associated with complex conjugate eigenvalues/eigenvectors pairs, the corresponding GFT coefficients satisfy $\hat{\vc x}[k_1]=\hat{\vc x}[k_2]^\star$.

In this case, filtering in the spectral domain is defined as 
\begin{equation}
    \vc y = \ma U \ma G\ma U^{-1}\vc x,
\end{equation}
with $\ma G\in \mathbb{C}^{N\times N}$ being a complex-valued diagonal matrix \cite{sandryhaila_discrete_2014}. 

\section{Stationary Signals on Graphs}
\label{sect:statio}
Stationarity of graph signals has been defined and studied on multiple occasions \cite{perraudin_stationary_2017, girault_translation_2015, iraji_wss_2025}. In this section, we revisit the definition of Graph Wide Sense Stationary (GWSS) through the eigendecomposition of the directed adjacency.

\subsection{Undirected Graph Stationary} 
Let $\vc x \in \mathbb{R}^N$ be a random multivariate variable that we further on refer to as a random graph signal. We can then define the first- and second-order statistical moments of $\vc x$ as  the mean $\vct \mu_{\vc x}[n]=\mathbb{E}[\vc x[n]]$ and the covariance matrix $\ma H_{\vc x}=\mathbb{E}[({\vc x}- \vct\mu_{\vc x})({\vc x}-\vct\mu_{\vc x})^H]$, respectively. Without loss of generality, we assume throughout the paper that $\vct\mu_{\vc x}=\vc 0$. While \cite{perraudin_stationary_2017} defines GWSS of $\vc x$ through conditions on mean and covariance, they also provide the following equivalent definition.
\begin{definition}[GWSS] \label{ustatio}
    A random graph signal $\vc x$ is GWSS on an undirected graph if and only if there exists $\hat{\bar{\ma H}}_\vc x$ such that its covariance matrix can be factorized as
    \begin{equation} 
        {\ma H}_{\vc x}= \bar{\ma U}\hat{\bar{\ma H}}_{\vc x}\bar{\ma U}^T,
    \end{equation}
where $\hat{\bar{\ma H}}_{\vc x}\in\mathbb{R}^{N\times N}$ is diagonal. In particular, the graph power spectral density (PSD) of $\vc x$ is defined as $\hat{\bar{\vc h}}_{\vc x}\in\mathbb{R}^N$: 
\begin{align*}
    \hat{\bar{\vc h}}_{\vc x}[n] = \hat{\bar{\ma H}}_{\vc x}[n, n],
\end{align*}
diagonal entries of $\hat{\bar{\ma H}}_{\vc x}$. 
\end{definition}

\begin{remark}
Furthermore, by rewriting Eq.~(\ref{ustatio}), we obtain 
\begin{equation}
    \hat{\bar{\ma H}}_{\vc x} = \bar{\ma U}^T{\ma H}_{\vc x} \bar{\ma U} = \mathbb{E}[ \hat{\vc x} \hat{\vc x}^H]
\end{equation}
being the graph spectral covariance matrix. Now since $\hat{\bar{\ma H}}_{\vc x}$ is diagonal, the spectral components of $\vc x$ are uncorrelated. When $\vc x$ is not stationary, $\hat{\bar{\ma H}}_{\vc x}$ is still the graph spectral covariance, however, it is not diagonal anymore.
\end{remark}

Given that the diagonal structure of $\hat{{\ma H}}_{\vc x}$ is a key property of a GWSS. Characterizing how diagonal the graph spectral covariance matrix is, can be used to characterize the stationarity of the random variable. Following this, we define the ``stationary level''~\cite{perraudin_stationary_2017}.

\begin{definition}[GWSS Stationary Level] \label{stationary-level-undirected}
For a random variable $\vc x$, its covariance matrix ${\ma H}_\vc x$, the associated graph spectral covariance matrix $\hat{\bar{\ma H}}_\vc x$ and the PSD $\hat{\bar{\vc h}}_\vc x$, the stationary level of $\vc x$ is defined as
    \begin{equation}
        \bar{\kappa}({\ma H}_\vc x) = \frac{||\hat{\bar{\vc h}}_\vc x||_2}{||\hat{\bar{\ma H}}_\vc x||_F}.
    \end{equation}
\end{definition}

In practice, we propose estimators for the covariance matrix and the graph spectral covariance matrix as follows. Given $K$ realizations of the random graph signal $\vc x$, denoted as $\vc m_k, k=1,\ldots,K$, we define the empirical covariance matrix $\ma H_\vc m$ as
\begin{equation}
    \ma H_\vc m = \frac{1}{K}\sum_{k=1}^K \vc m_k \vc m_k^H,
\end{equation}
and the empirical graph spectral covariance matrix $\hat{\ma H}_\vc m$ as
\begin{equation}
    \hat{\bar{\ma H}}_\vc m = \bar{\ma U}^H \ma H_\vc m \bar{\ma U} = \frac{1}{K}\sum_{k=1}^K \hat{\vc m}_k \hat{\vc m}_k^H,
\end{equation}
with $\hat{\vc m}_k = \bar{\ma U}^T \vc m_k$ being the GFT of $\vc m_k$. It follows that the empirical PSD is defined as $\hat{\bar{\vc h}}_\vc m[k] = \hat{\bar{\ma H}}_\vc m[k,k]$. Furthermore , we obtain the empirical stationary level $\bar{\kappa}(\ma H_\vc m)$ by replacing $\ma H_\vc x$ and $\hat{\bar{\ma H}}_\vc x$ in Definition~\ref{stationary-level-undirected} with their empirical counterparts $\ma H_\vc m$ and $\hat{\bar{\ma H}}_\vc m$, respectively. 

\begin{figure*}[hbt!]
    \centering
\captionsetup[subfigure]{justification=centering}
  \subfloat[Bi-cycle graph\\WN variances\label{auto_graph1}]{%
        \includegraphics[width=0.29\linewidth]{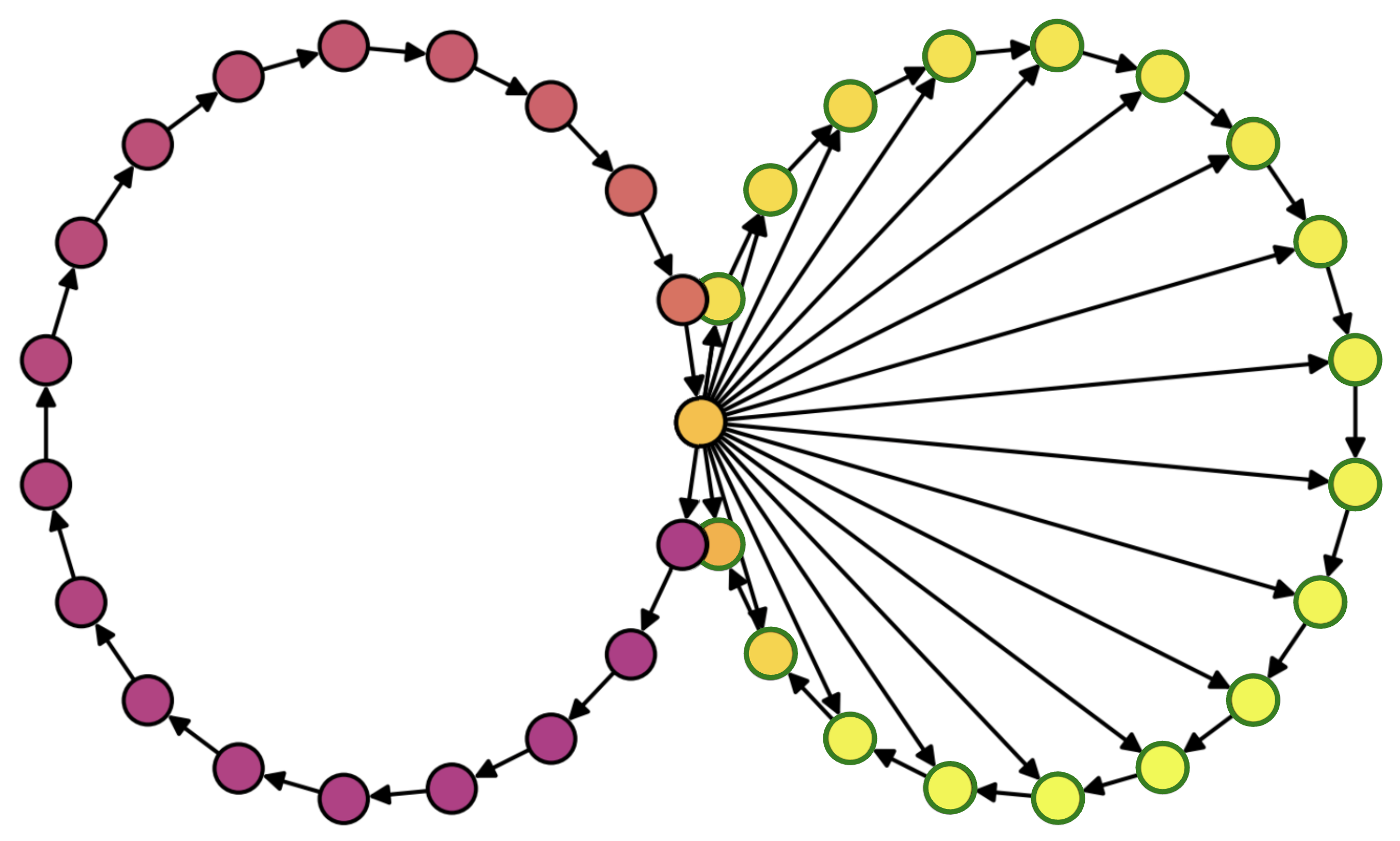}}
    \hfill
  \subfloat[Bi-cycle graph\\WN covariance matrix\label{wncovar_graph1}]{%
       \includegraphics[width=0.218\linewidth]{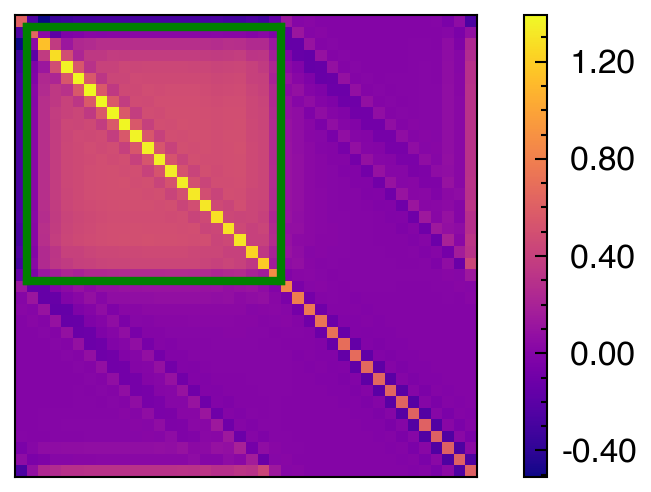}}
    \hfill
  \subfloat[USA graph\\WN variances\label{auto_graph2}]{%
        \includegraphics[width=0.26\linewidth]{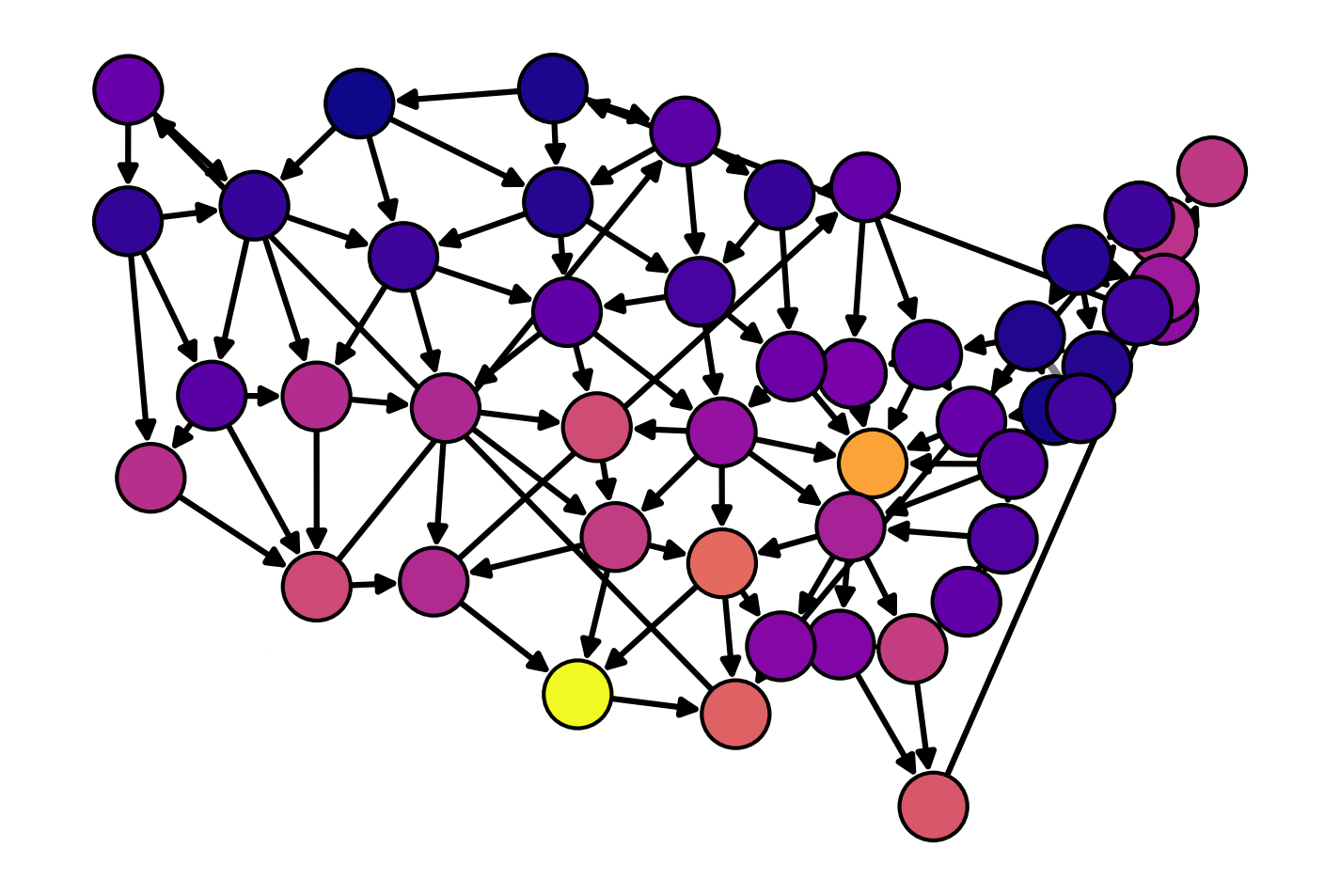}}
    \hfill
          \subfloat[USA Graph\\WN covariance matrix\label{wncovar_graph2}]{%
        \includegraphics[width=0.222\linewidth]{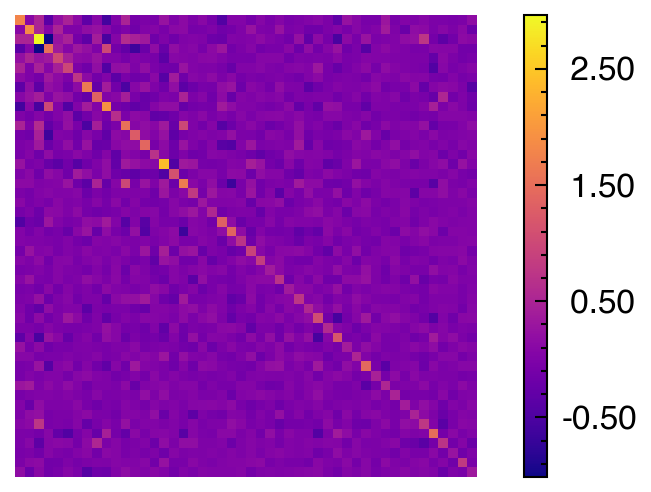}}
  \caption{\label{fig:cov} Examples of WN on directed graphs. (a)~Bi-cycle graph indicating nodal variances associated with WN. (b)~Covariance matrix associated with WN. The green rectangle corresponds to all nodes on the right cycle of the bi-cycle graph. 
 (c)~USA graph indicating nodal variances associated with WN. (d)~Covariance matrix associated with WN.}
\end{figure*}

\subsection{Directed Graph Stationary}
We first remind the notion of GWSS processes on directed graph introduced by \cite{iraji_wss_2025}. Then, we revisit the concept through the eigendecomposition of the graph shift operator instead of the JNF. We provide the definitions and properties of stationarity to directed graphs.
\begin{definition}[GWSS - Jordan Decomposition]
    Let $\ma A=\ma V \ma J\ma V^{-1}$ be the Jordan Decomposition of the adjacency matrix $\ma A$. The authors of \cite{iraji_wss_2025} state that a random graph signal is GWSS for the JNF of $\ma A$ if and only if there is a Jordan block diagonal matrix $\hat{\ma H}^{\ma J}_{\vc x}$ that relates to the covariance matrix $\ma H_{\vc x}$ as 
    \begin{equation}
        \ma H_{\vc x} = \ma V\hat{\ma H}^{\raisebox{-0.76ex}{$\scriptstyle \ma J$}}_{\vc x}\ma V^H.
    \end{equation}
\end{definition} Reformulating this definition for $\ma A$ diagonalizable leads to the generalized eigenvectors being replaced with the proper eigenvectors, and the block diagonal matrix turns into a regular diagonal matrix, which brings us to the following definition of DGWSS.
\begin{definition}[DGWSS] \label{dstatio}
    A random graph signal is Directed GWSS (DGWSS) if and only if there exists a diagonal matrix $\hat{\ma H}_{\vc x}\in\mathbb{C}^{N\times N}$ such that its covariance matrix can be expanded into:
    \begin{equation}
    \ma H_{\vc x}= \ma U\hat{\ma H}_{\vc x}\ma U^H. 
    \end{equation}
The diagonal entries of $\hat{\ma H}_{\vc x}$ contain the graph PSD $\hat{\vc h}_{\vc x}\in\mathbb{R}^N$: 
\begin{equation}
    \hat{\vc h}_{\vc x}[n] = \hat{\ma H}_{\vc x}[n, n].
\end{equation}
\end{definition}
Since $\hat{\ma H}_{\vc x}$ is diagonal, the spectral components of $\vc x$ are uncorrelated.

Given the definition of DGWSS, we revisit the following property of interest, enunciated by Proposition~\ref{filter-PSD}.
\begin{prop} \label{filter-PSD}
    Let $\ma G$ be a spectral filter and ${\vc x}\in \mathbb{R}^N$ a DGWSS random graph signal. The filtered random variable $\vc y$ by  the spectral filter $\ma G$ remains DGWSS and its PSD is given by 
    \begin{align*}
        \hat{\vc h}_{\vc y}[n] = |\ma G[n,n]|^2 \hat{\vc h}_{\vc x}[n].
    \end{align*}
\end{prop}
\begin{proof}
Given that \(\vc y=\ma U \ma G \ma U^{-1}\vc x\), we can write
\begin{align*}
\mathbb{E}[\vc y\vc y^H]
&= \ma U \ma G \ma U^{-1}\,\mathbb{E}[\vc x\vc x^H]\,\ma U^{-H}\ma G^H\ma U^H\\
&= \ma U\Bigl(\ma G\underbrace{\ma U^{-1}{\ma H}_{\vc x}\ma U^{-H}}_{=\;\hat{\ma H}_{\vc x}}\ma G^H\Bigr)\ma U^H\\
&= \ma U\bigl(\ma G\hat{\ma H}_{\vc x}\ma G^H\bigr)\ma U^H
= \ma U\bigl(|\ma G|^2\,\hat{\ma H}_{\vc x}\bigr)\ma U^H,
\end{align*}
which shows that \(\vc y\) is DGWSS and yields the PSD.
\end{proof}

Similar to Definition~\ref{stationary-level-undirected},  we give the definition of stationarity level for DGWSS. 

\begin{definition}[DGWSS Stationary Level] \label{stationary-level-directed}
For a random variable $\vc x$, its covariance matrix $\ma H_\vc x$, its associated graph spectral covariance matrix $\hat{\ma H}_\vc x$ and the associated PSD $\hat{\vc h}_\vc x$, the stationary level is defined as
    \begin{equation}
        \kappa(\ma H_\vc x) = \frac{||\hat{\vc h}_\vc x||_2}{||\hat{\ma H}_\vc x||_F}.
    \end{equation}
\end{definition}
An estimate $\kappa(\ma H_\vc m)$ for an empirical signal can be obtained by plugging in the empirical spectral covariance matrix $\hat{\ma H}_{\vc m}$ and PSD $\hat{\vc h}_{\vc m}$.

Unlike the undirected setting, the directed GFT basis is generally non-orthogonal, implying that directed-graph WN is no longer spatially uncorrelated. Consequently, node signals exhibit graph-dependent variances and cross-covariances determined by the graph topology and edge orientation.

Fig.\ref{fig:cov} illustrates the covariance structure of WN for two directed graphs. In the synthetic example (Fig.\ref{wncovar_graph1}), nodes sharing common predecessor nodes exhibit increased covariance, producing visible block structures in the covariance matrix. Moreover, nodes with larger in-degree present larger variances. Similar behavior can also be observed in the real graph example (Fig.~\ref{wncovar_graph2}), where covariance and variance patterns reflect the directional organization of the graph.

These observations highlight a fundamental distinction between undirected and directed graph WN. While WN on undirected graphs is graph-independent, directed-graph WN intrinsically encodes topology-dependent spatial correlations through the non-orthogonality of the directed GFT basis. This distinction will later play an important role in interpreting stationary surrogate signals and the associated null hypothesis, since preserving graph spectral second-order statistics on directed graphs also implicitly preserves graph-dependent covariance structures.

\section{White Noise on Graphs}
\label{sect:white-noise}
An important family of stationary random graph signals is white noise (WN). While WN on undirected graphs has similar properties to conventional WN on the Euclidean domain, its extension to directed graphs comes with some surprising properties. We will consider some practical examples to illustrate these points. 

\subsection{White Noise on Undirected Graph}
We consider a WN random graph signal $\vc w\in\mathbb{R}^N$ defined through its graph spectral covariance matrix being the identity matrix \cite[Chap. 8.10]{shynk_probability_2013}; i.e., $\hat{\bar{\ma H}}_{\vc w} = \ma I_N$. 
For an undirected graph, the GFT is a real-valued orthonormal basis, and thus the covariance matrix of $\vc w$ is also the identity matrix 
\begin{equation}
    \ma H_{\vc w} = \bar{\ma U}\hat{\bar{\ma H}}_{\vc w}\bar{\ma U}^H=\ma I_N.
\end{equation} Therefore, the variance is uniform for all nodes and the cross variances are null. In the case of a Gaussian distribution, the nodal signals are independent and identically distributed (i.i.d.). 

\subsection{White Noise on Directed Graph}
\label{wn-noise-directed}
The real-valued WN random graph signal $\vc w\in\mathbb{R}^N$ is also defined through the diagonal graph spectral covariance matrix $\hat{\ma H}_{\vc w}=\ma I_N$, and thus the following covariance matrix is obtained: 
\begin{equation}
    \ma H_{\vc w} = \ma U\hat{\ma H}_{\vc w}\ma U^H= \ma U\ma U^H.
\end{equation}
It should be noted that, despite $\ma U$ being a matrix of complex-valued eigenvectors, $\ma H_{\vc w}$ is a proper covariance matrix for $\vc w$, that it is real-valued and positive semi-definite as shown in the following proposition. 
\begin{prop}
    The matrix $\ma U\ma U^H$ is real-valued and positive semi-definite.
\end{prop}
\begin{proof}
    (Proof of realness) Rewrite $\ma U\ma U^H=\sum_k \vc u_k\vc u_k^H$ with $\vc u_k$ the eigenvectors. Now by grouping the indices of real eigenvalues into the set $\mathcal{R}$ and the indices of complex conjugate eigenvalues into the set $\mathcal{C}$, we have 
    \begin{align*}
        \ma U\ma U^H = \sum_{k=1}^N \vc u_k\vc u_k^H &= \sum_{k\in \mathcal{R}} \vc u_k \vc u_k^H + \frac{1}{2}\sum_{k\in \mathcal{C}} (\vc u_k \vc u_k^H + \vc u_k \conjugate \vc {u_k \conjugate}^H )\\
        &= \sum_{k\in \mathcal{R}} \vc u_k \vc u_k^H + \sum_{k\in \mathcal{C}} \Re(\vc u_k \vc u_k^H) \in \mathbb{R}^{N\times N}.
    \end{align*}
Thus proving realness of $\ma U\ma U^H$.\\
    (Proof of positive semi-definiteness) For any $\vc x\in\mathbb{R}^N$ we have \begin{align*}
        \vc x^H\ma U\ma U^H \vc x = \|\ma U^H\vc x \|_2^2 \geq 0
    \end{align*}
thus proving positive semi-definiteness of $\ma U\ma U^H$.
\end{proof}

Since the GFT for directed graphs is not an orthonormal basis, the WN covariance is generally not diagonal implying that the nodal signals carry different variances and are spatially correlated. In Fig.~\ref{fig:cov}, the WN variances and covariances are illustrated for two directed graphs. As opposed to WN for undirected graphs, WN on directed graphs becomes graph-specific.

\begin{figure*}[hbt!]
    \centering
  \subfloat[Set of irregular nodes\label{dirac-signal}]{%
       \includegraphics[width=0.34\linewidth]{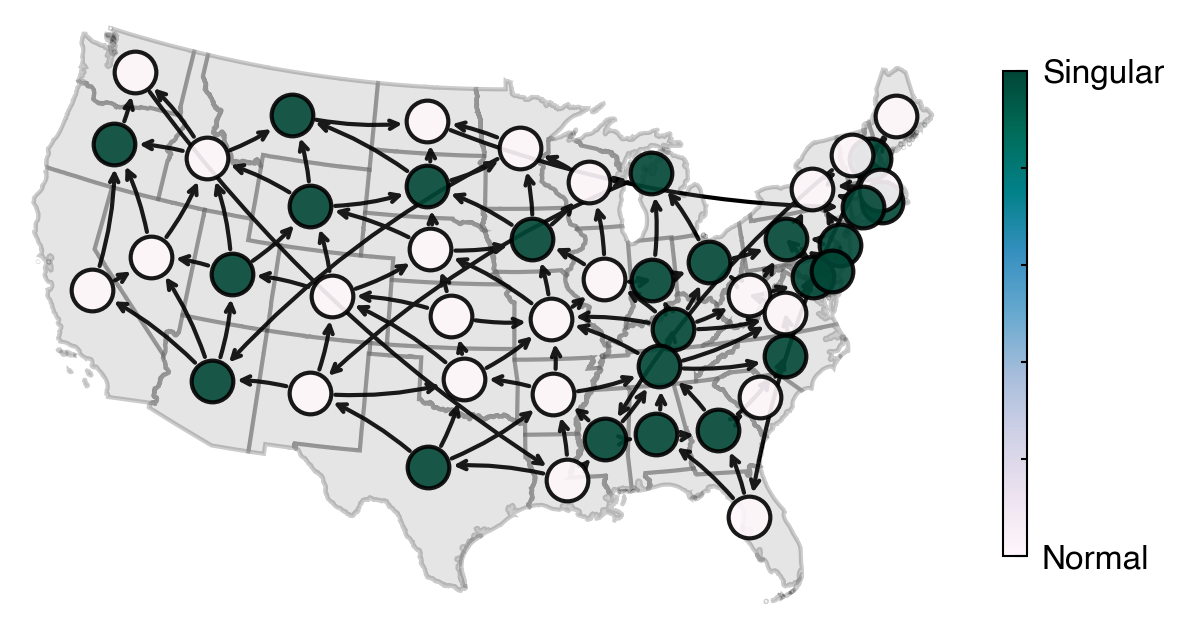}}
    \hfill
  \subfloat[Statistical significance \label{stats-dirac}]{%
        \includegraphics[width=0.32\linewidth]{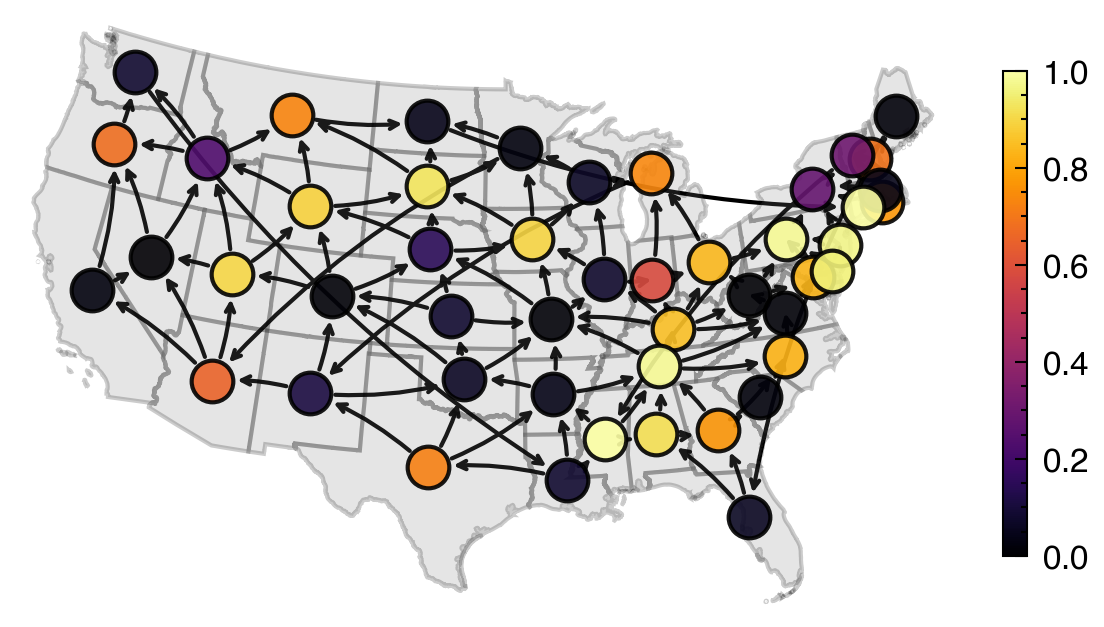}}
    \hfill
    \subfloat[Accuracies \label{stats-multi}]{%
        \includegraphics[width=0.32\linewidth]{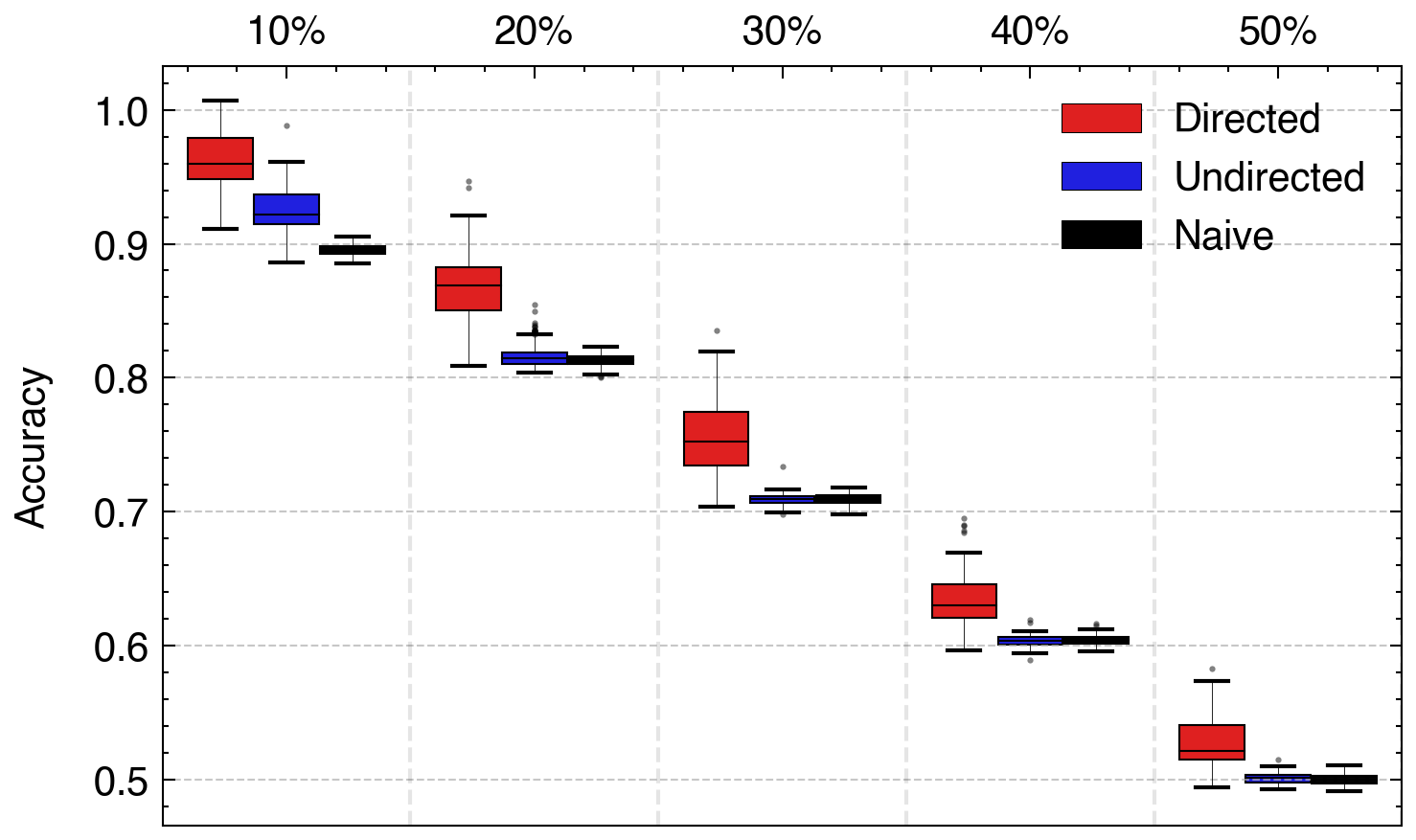}}
  \caption{(a)~Set of nodes on the USA graph with irregular signal contributions. (b)~Nodal statistical significance $(1-p\text{-value})$ obtained from directed-graph surrogates. (c)~Accuracies for different surrogate types (directed graph, undirected graph, naive permutation) and $50$ random sets of irregular nodes with different densities.}
\end{figure*}

\subsection{Illustrating Example}
We observe that the covariance of a directed graph WN is not necessarily diagonal \replaced{(Figs.~\ref{wncovar_graph1} and~\ref{wncovar_graph2}), reflecting the structure and orientation of the graph.}{. Here, we compare  the patterns of covariance with the graph structure.}\replaced{ For the synthetic graph in Fig.~\ref{auto_graph1}, the corresponding covariance matrix in Fig.~\ref{wncovar_graph1} exhibits block-like covariance structure between nodes sharing common parent nodes (indicated in green in Fig.~\ref{wncovar_graph1}). Intuitively, child nodes share a common confounding variable (parent node), resulting in increased cross-covariance. Moreover, nodes with larger strict in-degree, such as those on the right side of the bi-cycle in Fig.~\ref{auto_graph1}, show increased variances in Fig.~\ref{wncovar_graph1}. This suggests that variability accumulates through multiple incoming directional connections, whereas symmetric edges do not contribute additional variance because the variability would be shared. }{We first take as an example the graph in Fig.~\ref{auto_graph1}. Considering the WN specific to this graph, we look at its covariance matrix (Fig.~\ref{wncovar_graph1}) and see that high covariance between two nodes occurs when they share the same parent node, this explains the block like covariance (within the highlighted green Fig.~\ref{wncovar_graph1}). High covariance between children nodes can intuitively be interpreted as the children nodes sharing a common confound variable (parent node).

Once again, taking as an example the first graph (Fig.~\ref{auto_graph1}), higher auto-covariance is obtained when nodes have high numbers of strict in-degree connections. In Fig.~\ref{wncovar_graph1}, this is what we see for nodes on the right side of the bi-cycle, as they have $2$ strict in-degree connections in comparison to $1$ for all other nodes. Intuitively, higher number of in-degree connections hints to higher variance since in addition to the node's inherent variance, additional variability is created by parent nodes. Symmetric edges does not contribute to the variance as variability would be shared.} \replaced{Similar behavior can also be observed for the real graph in Fig.\ref{auto_graph2}. Although covariance clusters are less visually pronounced in Fig.\ref{wncovar_graph2}, nodes located lower in the graph tend to exhibit larger variances. This is consistent with the graph’s overall downward directional organization, where lower nodes receive simultaneous influences from multiple upstream nodes, leading to increased variability.}{
While in real graph (Fig.~\ref{auto_graph2}), spatial clusters of cross covariance between nodes are less apparent, it remains that nodes' auto correlation can be observed. Lower located nodes are often of higher in-degree as the graph has inherently a downward direction, moving toward fewer nodes. Intuitively, lower located nodes are perturbed simultaneously by many source nodes, explaining higher nodes' variance for lower located nodes (Fig.~\ref{wncovar_graph2}).}

\section{Surrogate Generation Scheme}
\label{sect:surrogate}

Equipped with the notion of stationarity we present the scheme for surrogate generation, starting with the null hypothesis under which our surrogate graph signals are generated. Then we briefly remind surrogate generation on undirected graphs and generalize surrogate generation to directed graphs. We also prove that these surrogate processes stationarizes non-stationary signals such that the expectancy of the surrogates is a stationary version of the initial signal.

\begin{figure*}[hbt!]
    \centering
  \subfloat[Irregular covariance structure \label{gt-connectivity}]{%
       \includegraphics[width=0.21\linewidth]{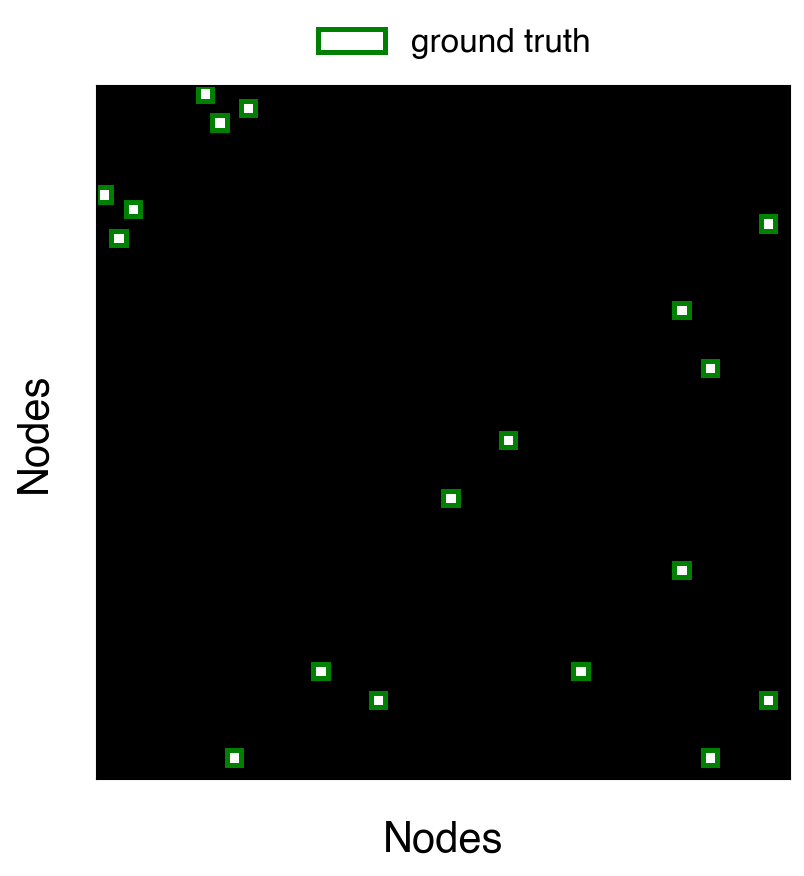}}
       \hfill
  \subfloat[Detections directed\label{test-connectivity-directed}]{%
        \includegraphics[width=0.21\linewidth]{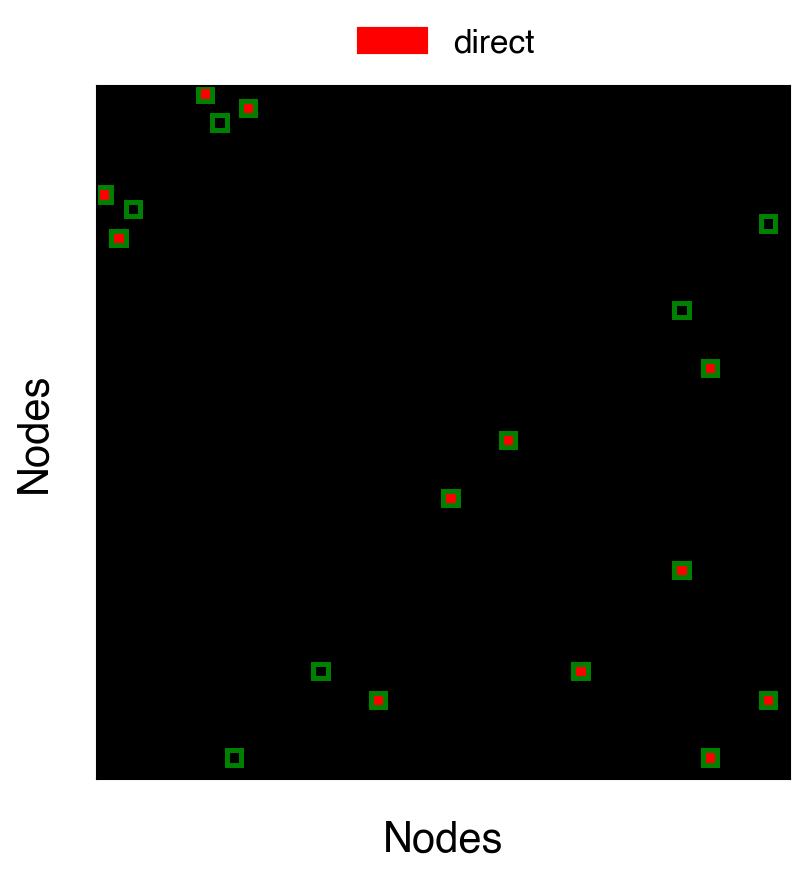}}
        \hfill
\subfloat[Detections undirected\label{test-connectivity-undirected}]{%
        \includegraphics[width=0.21\linewidth]{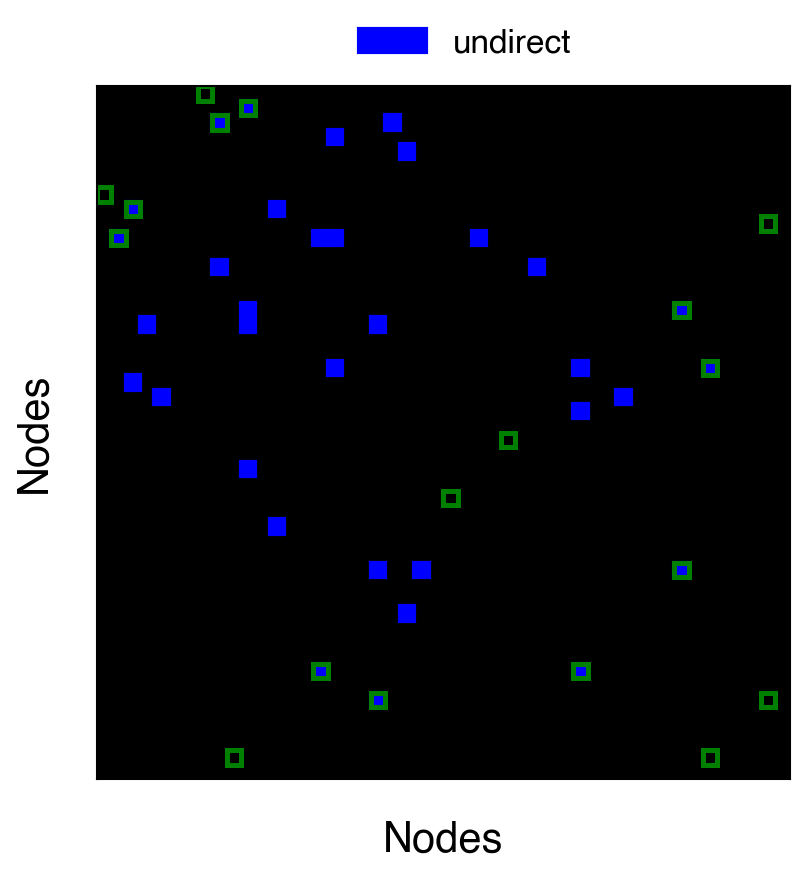}}
\subfloat[Accuracies\label{connectivity-distribution}]{%
  \raisebox{60pt}{
    \parbox{0.32\linewidth}{%
      \includegraphics[width=\linewidth]{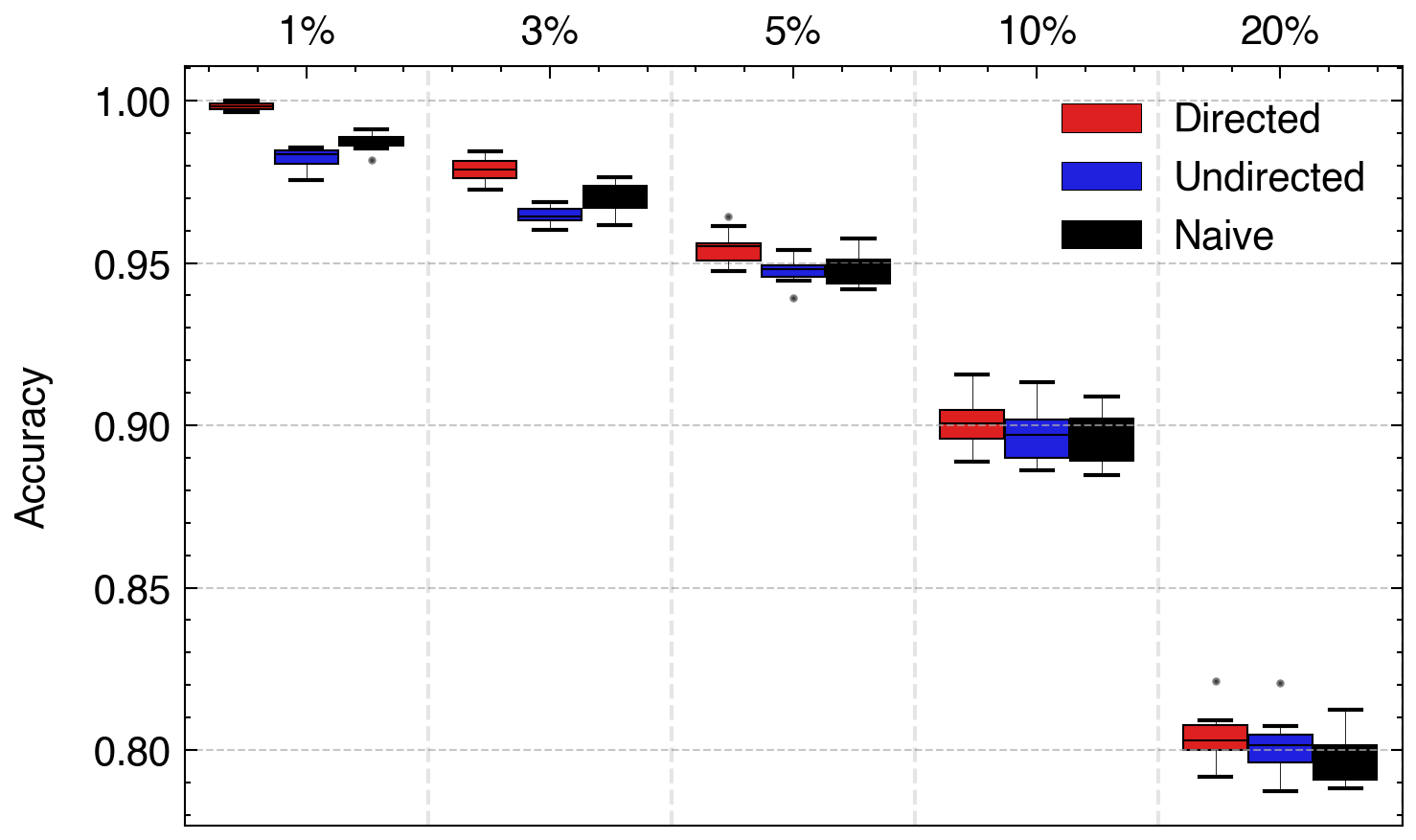}%
    }%
  }%
}
  \caption{(a)~Set of irregular edges. (b,c)~Detected pairs using directed-graph surrogates (red), undirected-graph surrogates (blue). Ground truth pairs in green frame. (d)~Accuracies for different surrogate types on $50$ random irregular covariances with different densities.}
  \label{fig5} 
\end{figure*}

\subsection{Null Hypothesis}
We consider as null hypothesis $\mathcal{H}_0$ that the observed graph signals $\vc m_k, k=1,...,K$, are realizations of a (D)GWSS $\vc x$ on a known graph. The surrogate generation scheme comes up with graph signals $\vc m_k^{(l)}$ under the assumption of  stationarity.

\subsection{Undirected Surrogate Graph Signals}
Following our previous findings on \cite{pirondini_spectral_2016}, we define the randomization filter $\bar{\ma R}$ as a diagonal sign randomization matrix
\begin{equation}
    \bar{\ma R}[n,n] = \epsilon_n,
\end{equation}
each $\epsilon_n$ being an i.i.d.\ Rademacher random variable \cite{montgomery1990distribution} with $\epsilon_n\in\{-1,1\}$, and $\mathbb{P}(\epsilon_n=1)=\mathbb{P}(\epsilon_n=-1)=0.5$. 
From a single measurement $\vc m$, we consider $L$ instances of the sign randomization matrix $\bar{\ma R}^{(l)}, l=1,\dots,L$, then $L$ surrogates using the undirected graph structure are generated as follows:
\begin{equation}
    \vc m^{(l)} = \bar{\ma U}\,\bar{\ma R}^{(l)}\bar{\ma U}^T \vc m.
\end{equation}
An important result of sign-randomization is that the undirected-graph surrogates are realizations of a GWSS process.

\begin{prop}\label{proof-undirected-statio}
    For an observed graph signal $\vc m\in\mathbb{R}^N$, the undirected graph surrogates $\vc m^{(l)}$ are realizations of a GWSS.
\end{prop}
\begin{proof}
Let the surrogate random variable be \(\vc z=\bar{\ma U}\bar{\ma R}\bar{\ma U}^T\vc m\). Its covariance equals
\begin{align*}
\ma H_{\vc z}=\mathbb{E}\big[\vc z\vc z^H\big]
= \bar{\ma U}\,\mathbb{E}\big[\bar{\ma R}\underbrace{\bar{\ma U}^T\vc m\vc m^H\bar{\ma U}}_{=\;\hat{\bar{\ma H}}_\vc m}\bar{\ma R}^T\big]\,\bar{\ma U}^T
= \bar{\ma U}\hat{\bar{\ma H}}_\vc z\bar{\ma U}^T,
\end{align*}
where \(\hat{\bar{\ma H}}_\vc z \coloneqq \mathbb{E}\big[\bar{\ma R}\hat{\bar{\ma H}}_\vc m\bar{\ma R}^T\big]\). Since \(\bar{\ma R}\) is diagonal with entries \(\bar{\ma R}[k,k]=\epsilon_k\) (i.i.d.\ Rademacher), the \((n,m)\)-entry of \(\hat{\bar{\ma H}}_\vc z\) is
\begin{align*}
\hat{\bar{\ma H}}_\vc z[n,m]
&= \mathbb{E}\big[\bar{\ma R}[n,n]\,\hat{\bar{\ma H}}_\vc m[n,m]\,\bar{\ma R}[m,m]\big] \\
&= \mathbb{E}[\epsilon_n\epsilon_m]\,\hat{\bar{\ma H}}_\vc m[n,m].
\end{align*}
Because \(\epsilon_n\) are independent with \(\mathbb{E}[\epsilon_n]=0\) and \(\mathbb{E}[\epsilon_n^2]=1\), we have \(\mathbb{E}[\epsilon_n\epsilon_m]=\delta_{n,m}\). Hence
\begin{align*}
\hat{\bar{\ma H}}_\vc z[n,m]=\delta_{n,m}\,\hat{\bar{\ma H}}_\vc m[n,m],
\end{align*}
so \(\hat{\bar{\ma H}}_\vc z\) is diagonal. Therefore \(\ma H_{\vc z}=\bar{\ma U}\hat{\bar{\ma H}}_\vc z\bar{\ma U}^T\) has the form required by Definition~\ref{ustatio}, and \(\vc m^{(l)}\) is a realization of a GWSS process.
\end{proof}

\subsection{Directed Surrogate Graph Signals}
Generalizing the sign randomization matrix, we present a phase randomization matrix $\ma R$ for directed graph surrogates, defined as follows:
\begin{equation}
    \ma R[n,n] =  
    \left\{ \begin{array}{ll} 
    e^{j\Psi_n}, & \Im(\lambda_n) \neq 0,\\
    \epsilon_n, & \Im(\lambda_n)=0,
    \end{array} \right.
\end{equation}
with $\epsilon_n$ being i.i.d Rademacher random variables and 
$\Psi_k \overset{\text{i.i.d}}{\sim} \mathcal{U}([-\pi, \pi])$. If a real signal is desired then for indices of conjugate pairs $k, l$, i.e where $\lambda_k = \lambda_l^*$, we set $\Psi_k=-\Psi_l$. 
In directed graphs, the eigenvectors are not necessarily orthogonal, and it is well known that due to non-orthogonality, Parseval's theorem does not hold in general \cite{seifert_digraph_2021}. As such energy is in general not preserved when applying a phase randomization. Yet, under the null hypothesis of stationarity, the energy is preserved. To show this, we first show Proposition~\ref{dphase-preserveauto}. 

\begin{prop} \label{dphase-preserveauto}
If $\vc x$ is DGWSS, then $\vc z=\ma U\ma R\ma U^{-1}\vc x$ is DGWSS and we have $\ma H_\vc z = \ma H_\vc x$.
\end{prop}
\begin{proof}
This is an immediate consequence of Proposition~\ref{filter-PSD}, since $|\ma R[n,n]|^2=1$ for $n=1,...,N$. 
\end{proof}
Following Proposition~\ref{dphase-preserveauto}, notice that from $\ma H_\vc z = \ma H_\vc x$, we have in particular $\Tr(\ma H_\vc z) = \Tr(\ma H_\vc x)$ which equivalently yields $||\vc x||_2=||\vc z||_2$. Thus a phase-randomized DGWSS is also DGWSS and preserves energy.

Under null hypothesis, a measured signal $\vc m$ and $K$ instances of the phase randomization matrix $\ma R^{(l)}, l=1,\dots,L$, then $L$ directed graph surrogates are generated as follows:
\begin{equation}
    \vc m^{(l)} = \ma U \ma R^{(l)} \ma U^{-1}\vc m.
\end{equation}
In practice, the energy of $\vc m^{(l)}$ may differ from that of $\vc m$. Therefore, we normalize the surrogate by a factor $||\vc m||_2/||\vc m^{(l)}||_2$. 

Finally, we show once more that the directed-graph surrogates are realizations of a DGWSS process.
\begin{prop}
    For an observed graph signal $\vc m\in\mathbb{R}^N$, the directed graph surrogate $\vc m^{(l)}$ is a realization of a DGWSS.
\end{prop}
\begin{proof}
Let $\vc z=\ma U\ma R\ma U^{-1}\vc m$ the surrogate random variable associated to $\vc m^{(l)}$. The covariance matrix of $\vc z$ is
\begin{align*}
\ma H_{\vc z} &= \mathbb{E}\big[\vc z\vc z^H\big] = \ma U\, \mathbb{E}\big[\ma R \underbrace{\ma U^{-1} {\vc m\vc m}^H\, \ma U^{-H}}_{=\hat{\ma H}_\vc m} \ma R^H\big]\ma U^H.
\end{align*}
Define the spectral covariance
\begin{align*}
\hat{\ma H}_{\vc z}
\coloneqq 
\mathbb{E}\big[\ma R\, \hat{\ma H}_{\vc m}\ma R^H\big],
\qquad 
\ma H_{\vc z}=\ma U\hat{\ma H}_{\vc z}\ma U^H.
\end{align*}
Since $\ma R$ is diagonal,
\begin{align*}
\hat{\ma H}_{\vc z}[n,m] = \mathbb{E}\big[\ma R[n,n]\, \ma R^*[m,m]\big]\, \hat{\ma H}_{\vc m}[n,m].
\end{align*}
Thus $\hat{\ma H}_{\vc z}$ is diagonal once we show $\mathbb{E}\big[\ma R[n,n]\ma R^*[m,m]\big]=\delta_{n,m}$. Each diagonal entry of $\ma R$ is either a Rademacher variable $\epsilon$ (real eigenvalue) or a random unit phase $e^{j\Psi}$ (complex eigenvalue), with conjugate pairs sharing $\Psi_l=-\Psi_k$. 
If $n=m$ then $|\ma R[n,n]|=1$ so the expectation equals $1$. For $n\neq m$, we have: 
\begin{itemize}
    \item $\mathbb{E}[\epsilon_n\epsilon_m]=0$ (independent real signs);
    \item $\mathbb{E}[e^{j(\Psi_n-\Psi_m)}]=0$ (independent phases);
    \item $\mathbb{E}[e^{2j\Psi_n}]=0$ (conjugate-phase pairs);
    \item $\mathbb{E}[\epsilon_n e^{-j\Psi_m}]
    =\mathbb{E}[\epsilon_n]\mathbb{E}[e^{-j\Psi_m}]=0$ (real/complex).
\end{itemize}
Hence we obtain
\begin{align*}
\mathbb{E}\big[\ma R[n,n]\, \ma R^*[m,m]\big]
=\delta_{n,m},
\end{align*}
so that $\hat{\ma H}_{\vc z}$ is diagonal and $\ma H_{\vc z} = \ma U\hat{\ma H}_{\vc z}\ma U^H$. Thus $\vc z$ is a DGWSS process.
\end{proof}

\subsection{Building Test Statistic Distributions with Surrogates}
Once surrogates are generated, test statistics need to be defined in order to verify the null hypothesis. In this work, we show two simple statistics that can be easily translated to a wide range of applications.
\subsubsection{Nodal Value} \label{node-test} From $K$ observations $\{\vc m_k\}_{k=1}^K$ we generate $L$ surrogate sets $\{{\vc m}_{k}^{(l)},k=1,\ldots,K\}_{l=1}^L$. The average empirical node signal $\vc m[n]=\frac{1}{K}\sum_{k=1}^K \vc m_k[n]$ is compared to the null distribution
\begin{equation}
\Pi(n)=\left\{\frac{1}{K}\sum_{k=1}^K {\vc m}_{k}^{(l)}[n], \quad l=1,\ldots,L\right\},
\end{equation}
identifying the $p$-value (one-tailed) under the null hypothesis as being $l'/(L+1)$ where $l'$ is the rank of $\vc m[n]$ among the sorted values of $\Pi(n)$.

This test procedure allows to detect node singularities that deviate from stationarity. Since we test for each node, an appropriate discount needs to be applied on the p-value. One could for instance use a Bonferroni correction \cite{dunn1961multiple} dividing the significance level $\alpha$ by $N$, but also implying that at least $N\lceil\frac{1}{\alpha}\rceil$ surrogates would need to be generated to make sure that the multiple-comparison adjusted significance level is reached.

\subsubsection{Covariance structure} \label{connect-test} Starting from $K$ observations $\vc m_k$, $k=1,\ldots,K$, we generate $L$ surrogate sets $\{{\vc m}_{k}^{(l)},k=1,\ldots,K\}_{l=1}^L$. We are interested in testing the relationship between pairs of nodes. To do so we compute the empirical covariance 
\begin{equation}
    \ma H_\vc m = \frac{1}{K} \sum_{k=1}^K \vc m_k \, \vc m_k^T
\end{equation}
and compare each entry $\ma H_\vc m[n,m]$ to the following null distribution
\begin{equation}
\Pi(n,m)=\left\{\frac{1}{K} \sum_{k=1}^K {\vc m}_{k}^{(l)}[n] \, {\vc m}_{k}^{(l)}[m], \quad l=1,\ldots,L\right\},
\end{equation}
identifying the $p$-value (one-tailed) under the null hypothesis as being $l'/(L+1)$ where $l'$ is the rank of $\ma H_\vc m[n,m]$ among the sorted values of $\Pi(n,m)$. 

\begin{figure*}
    \centering
    \captionsetup[subfigure]{justification=centering}
  \subfloat[Nodal timecourses\label{ar-n1-test}]{%
       \includegraphics[width=0.32\linewidth]{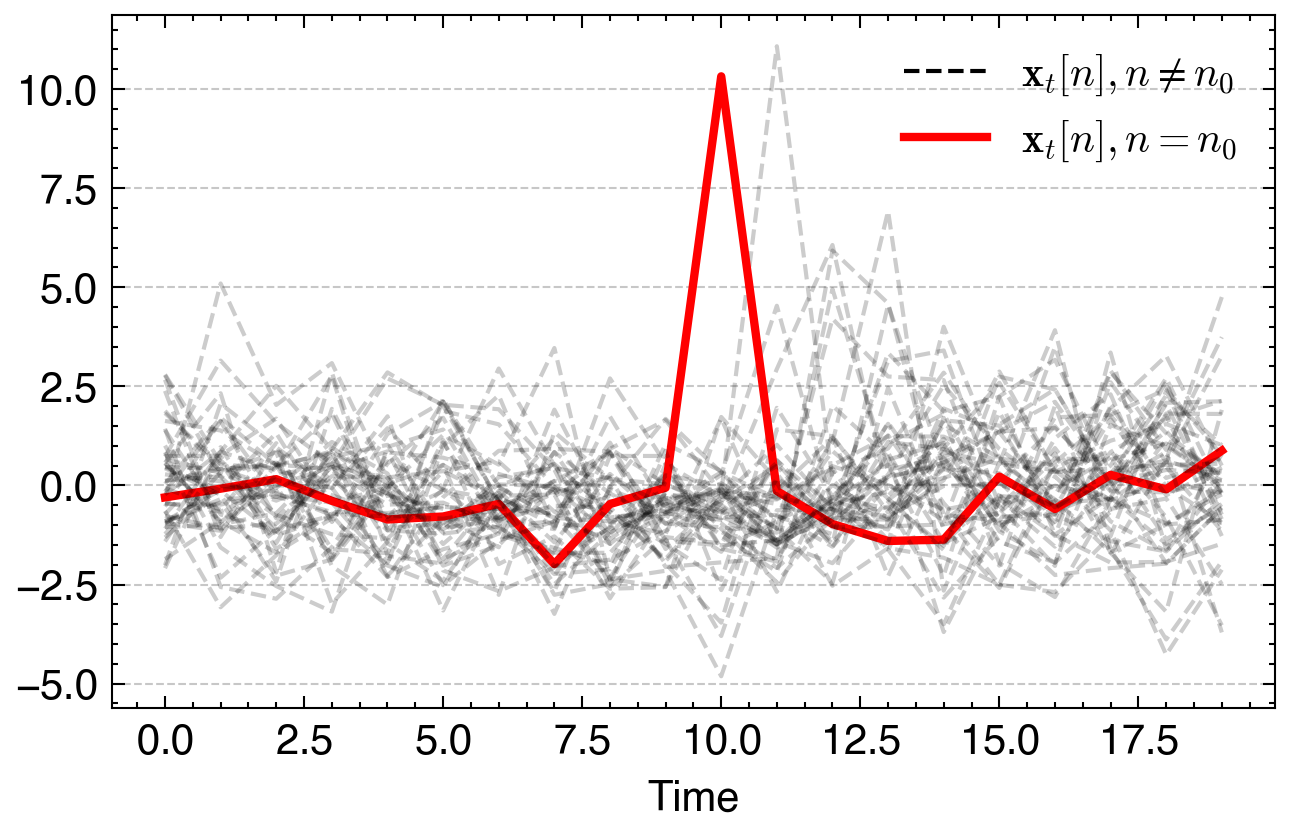}}
       \hfill
    \subfloat[Nodal null distributions \label{ar-n1-distrib}]{%
       \includegraphics[width=0.34\linewidth]{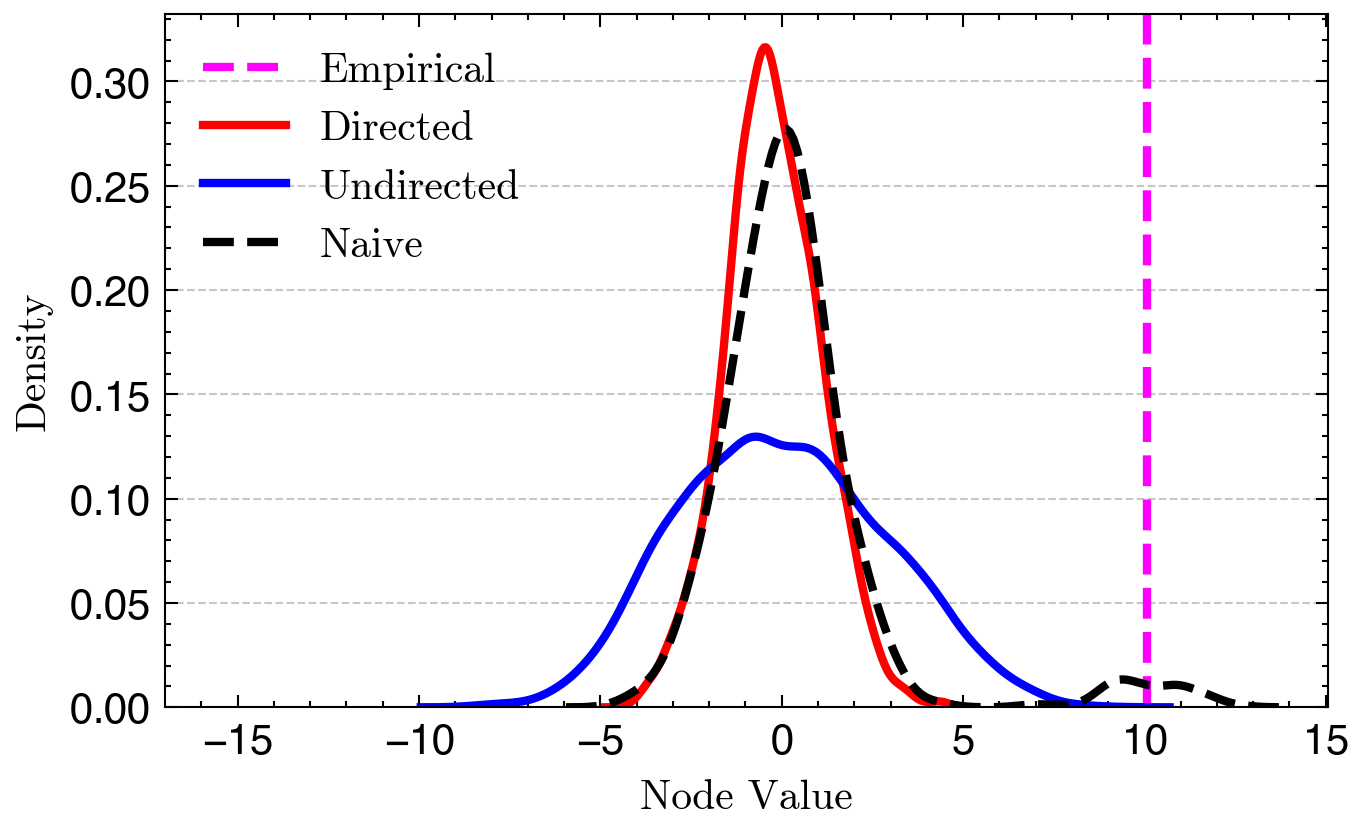}}
        \hfill
    \subfloat[ROC curves on $n_0$\label{ar-n1-roc}]{%
       \includegraphics[width=0.33\linewidth]{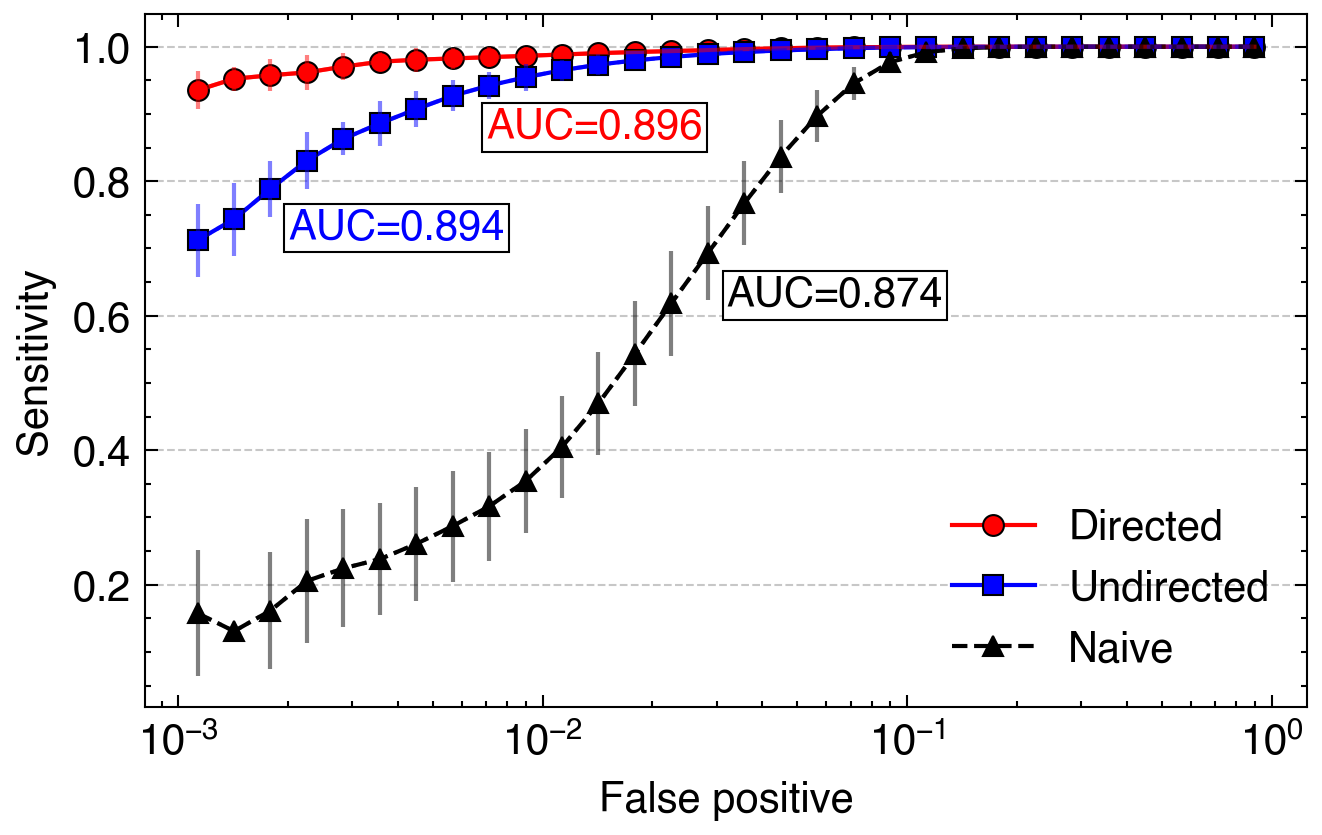}}
  \caption{(a)~Nodal timecourses of graph diffusion for different nodes with the irregular node ($n_0$) in red such that $\vct \delta_9[n_0]= 10$. (b)~Distribution of surrogate value for node $n_0$ at $t=9$, for one realization of the sequence. (c)~ROC curves for different $n_0$ with $\vct \delta_9[n_0]= 10$. Each curve results from $K=50$ repetitions and $L=999$ randomizations per repetition.}
\end{figure*}

\begin{figure*}
    \centering
    \captionsetup[subfigure]{justification=centering}
  \subfloat[Nodal timecourses\label{ar-multi-test}]{%
       \includegraphics[width=0.32\linewidth]{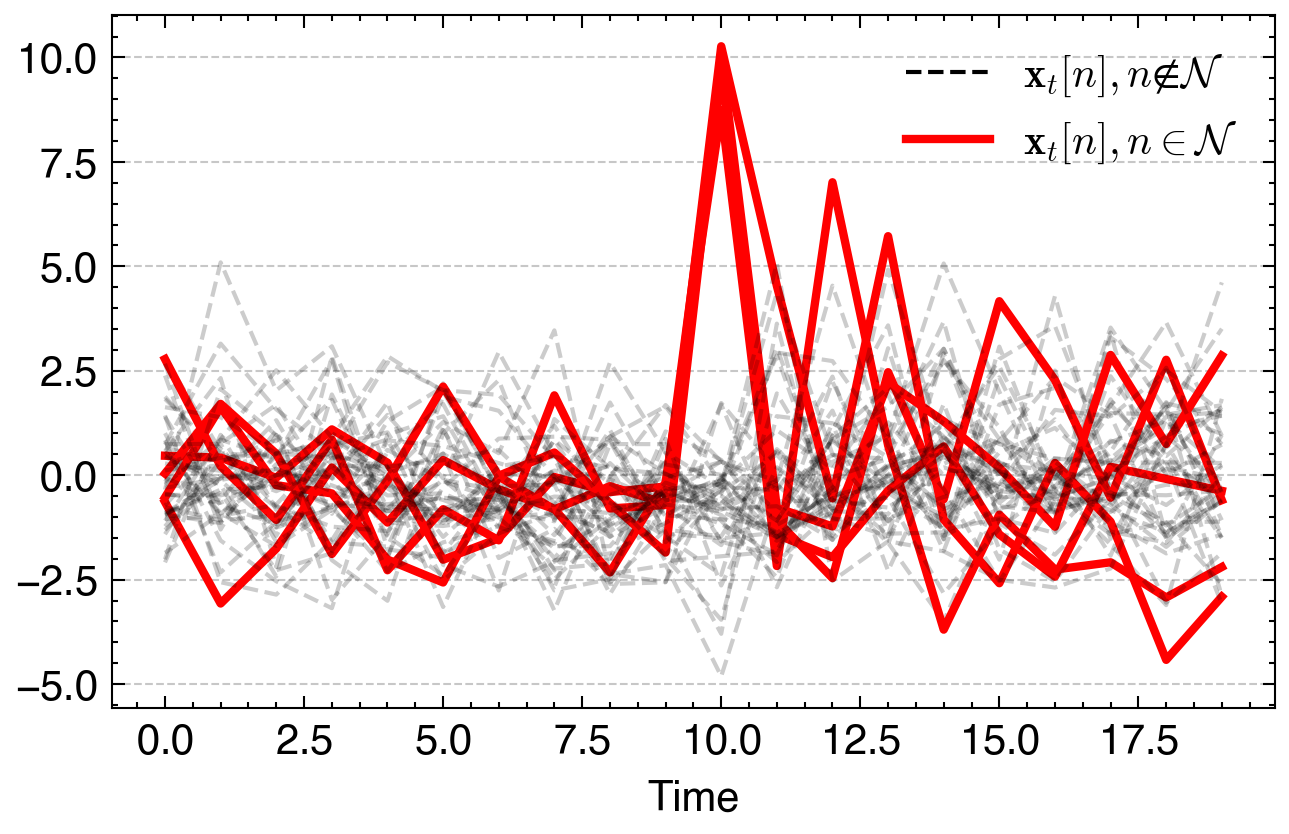}}
       \hfill
    \subfloat[ROC curves for $\text{card}(\mathcal{N})=5$ \label{ar-n5-roc}]{%
       \includegraphics[width=0.325\linewidth]{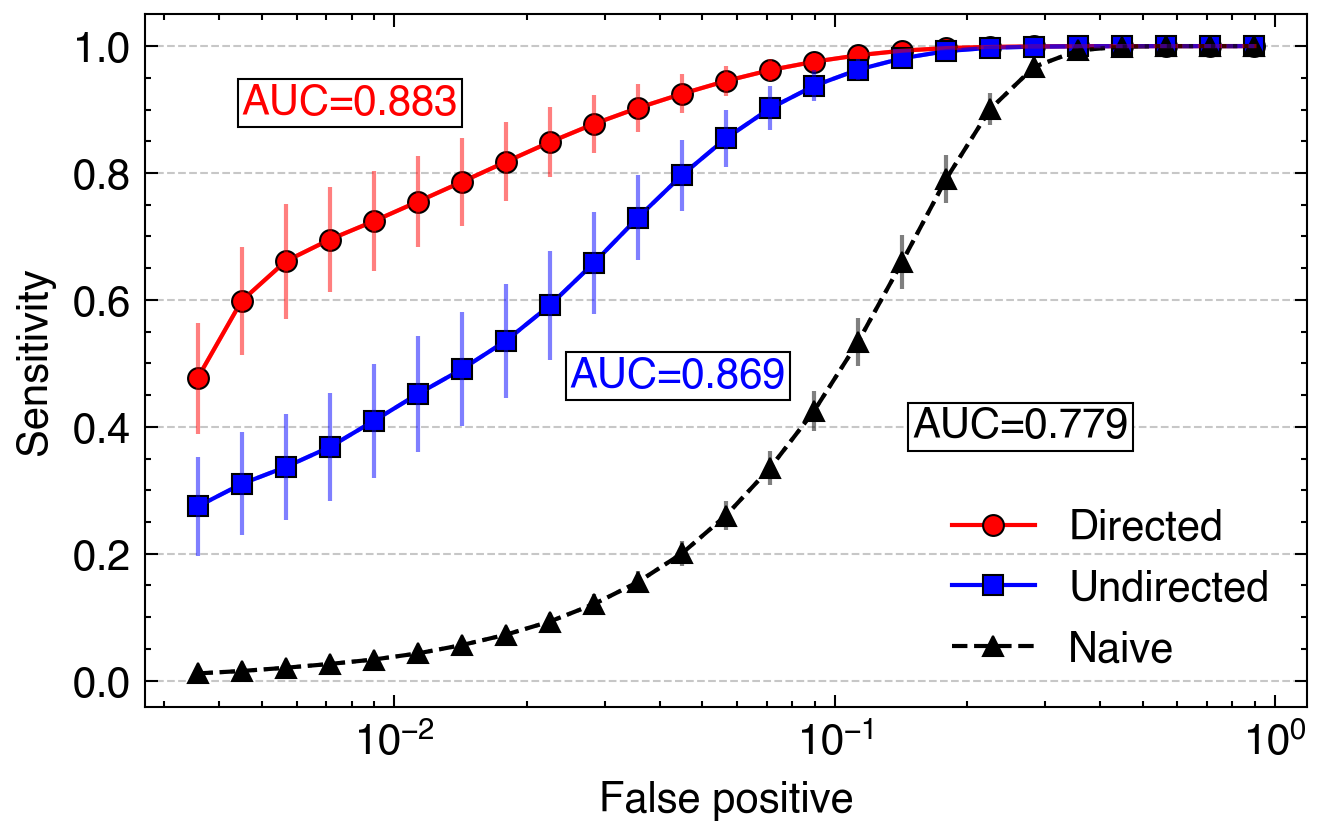}}
       \hfill
    \subfloat[AUC for different sizes of $\mathcal{N}$\label{auc_multiple}]{%
       \includegraphics[width=0.345\linewidth]{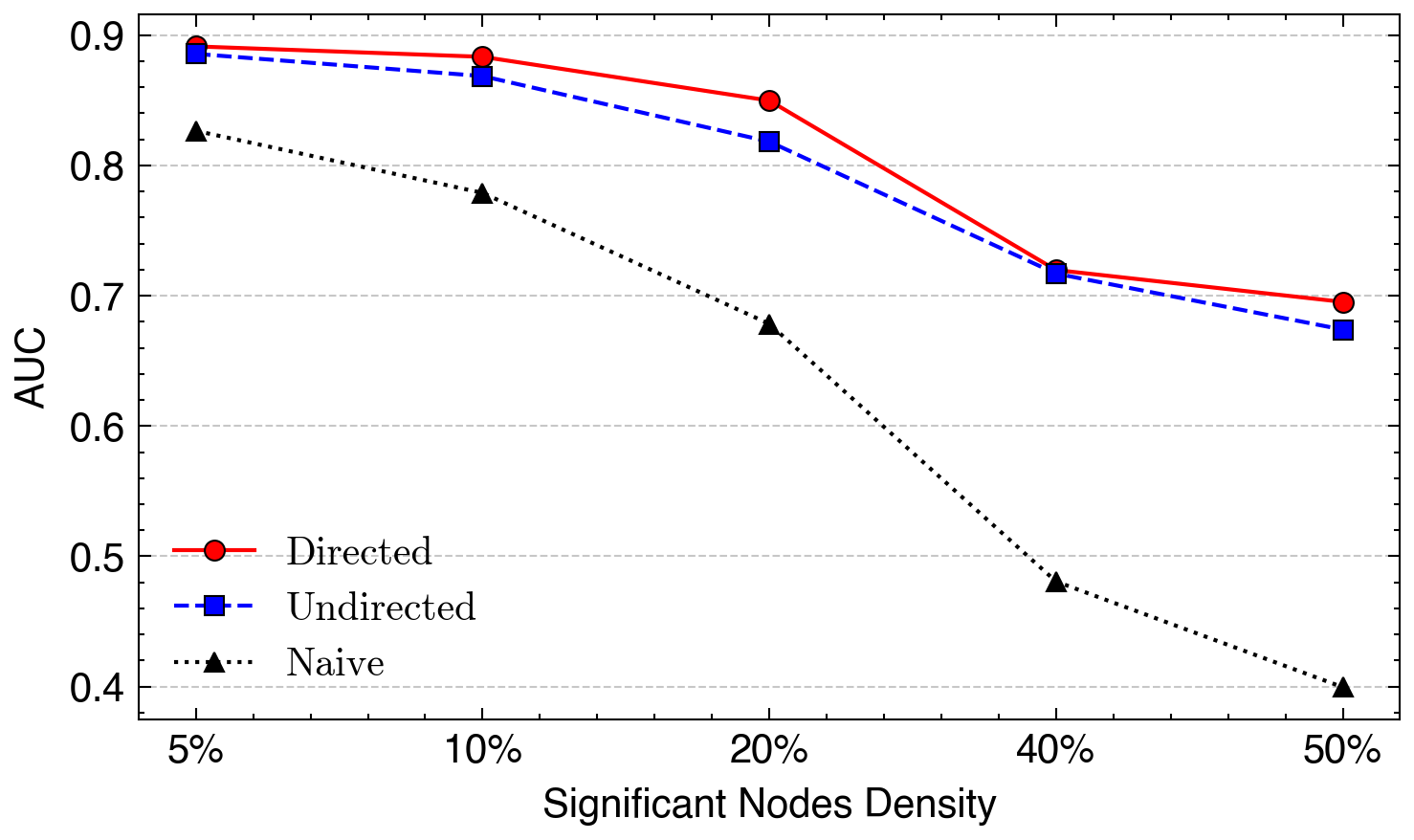}}
  \caption{(a)~Nodal timecourses of graph diffusion for a set of irregular nodes. (b)~ROC curves for different sets $\mathcal{N}$ of size $5$ with $\vct \delta_{9}[n]= 10$ for all $n\in \mathcal{N}$. Each curve results from $50$ repetitions and $999$ randomizations per repetition. (c)~AUC for various sizes of $\mathcal{N}$.}
\end{figure*}

This test allows to detect unexpected relationship between pairs of nodes. Similarly to the node value statistics, there is once again a need to compensate for the number of multi-comparisons. As the statistics is a symmetric function, we need to correct for $\left\lceil \frac{N(N-1)}{2} \right\rceil$ comparisons.
Following up, we give simple guiding examples for both of these tests.

\section{Experimental Results}
\label{sect:experimentals}
Having introduced the schemes for surrogate generation, we provide illustrative examples on their usage for directed and undirected graphs, compared with a baseline of naive permutation-based surrogates. The latter are obtained by uniformly randomly permuting entries of the original signal. As an immediate application of the previously introduced statistical tests, we detect nodal signals and covariances against a null of stationary graph signals. Then, in a more realistic setting, we consider timecourses generated from diffusion model \cite{tuerlinckx2001comparison} with additive disruptions on nodal signals. Throughout all three experiments, the employed graph will be the USA graph from \cite{seifert_digraph_2021, chan_hilbert_2025} \added{. For all experiments, the number of surrogates will be chosen such that the resulting p-value quantization, given by $1/(L+1)$, remains smaller than the significance level adjusted for multiple comparison and the initial significance level is conventionally set to $\alpha=5\%$ \cite{fisher1970statistical}.} Finally, we leverage our surrogate scheme in a real-world setting, with real graphs and graph signals, based on wind and temperature maps.
\subsection{Detection of Non-Stationary Signal Contributions} \label{test-node}
We address the detection of non-stationary nodal graph signal contributions by considering $K$ observations $\{ \vc m_k \}_{k=1,\ldots,K}$ of the following signal generation model:
\begin{equation}
    \vc x[n] = \vct \epsilon[n] + \vct \delta[n],
\end{equation}
with $\vct \epsilon$ is WN on a directed graph, and $\vct \delta$ deterministic signal contributions on ``irregular nodes'' that do not obey the DGWSS: 
\begin{equation}
    \vct \delta[n] = \left\{ \begin{array}{ll} 
    a, & n\in\mathcal{N},\\
    0, & n\notin\mathcal{N},
    \end{array} \right.
\end{equation}
where $\mathcal{N}$ is a subset of nodes. In our experiment, we set $a=10$ while $\vct \epsilon$ is the WN defined as in Sect.~\ref{wn-noise-directed} with respect to the directed graph. For the USA graph considered here, we have $\text{Var}(\vct \epsilon[n])\leq 3$ as shown in Fig.~\ref{wncovar_graph2}.

We first illustrate the detection on an arbitrary set of indices $\mathcal{N}$ with $\text{card}(\mathcal{N})=15$ as shown in Fig.~\ref{dirac-signal}. For each of the observations $\vc m_k$ ($K=20$), we generate $4999$ surrogate signals according to the proposed method for directed graphs, undirected graphs, or naive, leading to surrogate sets $\{{\vc m}_{k}^{(l)}, k=1,...,20\}_{l=1}^{4999}$. \added{A non-round number of surrogates (e.g., $L=4999$) is intentionally selected since including the observed statistic in the ranking leads to convenient p-value quantization levels of $1/(L+1)$. Similar choices of $L$ are adopted in the subsequent experiments.} We proceed by testing per node as described in Sect.~\ref{node-test} to obtain a nodal null distribution. When the empirical nodal signal has rank $t$ according to the null, its statistical significance is given by $1-t/(T+1)$ (Fig.~\ref{stats-dirac}). We verify that nodes with strongest significance, that is with the lowest p-values, belong to $\mathcal{N}$. We also observe a posteriori, that the least significant node in $\mathcal{N}$ ($=0.58$) is still more significant than the most significant node not in $\mathcal{N}$ ($=0.31$), implying perfect distinction. Results of the statistical significance for undirected-graph and permutation-based surrogates can be found in supplementary materials.

We next consider multiple instances of sets $\mathcal{N}$ with cardinality ranging from $10\%$ to $50\%$ of the number of nodes. Using the same number of surrogates and observations as in the illustrative example, we build a distribution under the null hypothesis as in Sect.~\ref{node-test} using different types of surrogates: for directed and undirected graphs, as well as naive permutation. We evaluate the performance of each surrogates using the accuracy score, defined as the proportion of correct predictions. The results are summarized in Fig.~\ref{stats-multi}\replaced{ and are corrected for multiple comparisons.}{. The initial threshold used for statistical significance is $\alpha=5\%$, which is corrected for multiple comparisons.} \added{For this and the following experiments, the results reported in the main text use Bonferroni correction. Results obtained using false discovery rate (FDR)\cite{benjamini1995controlling} correction are also provided in Supplementary Material.} Accuracy for directed graphs outperforms both undirected and naive. This short experiment validates the viability of surrogate on directed graphs. 

\subsection{Detection of Non-Stationary Covariance Structure}
\label{test-covar}
We now consider the following signal generation model 
\begin{equation}
    \vc x[n] \;=\; \vct \epsilon[n] + \vct \eta[n],
\end{equation}
where $\vct \epsilon$ is again WN on the known directed graph, and $\vct \eta$ of a zero-mean random process with a sparse covariance matrix $\ma H_{\vct \eta}$ where some entries are $1$ and thus are not DGWSS. Since the two processes are assumed to be uncorrelated, we have
\begin{equation}
    \ma H_\vc x \;=\; \ma H_{\vc \epsilon} + \ma H_{\vct \eta}.
\end{equation}
The signal realizations are denoted again as $\vc m_k$, $k=1,\ldots,K$. 

In Fig.~\ref{gt-connectivity}, we show a ground truth covariance matrix $\ma H_{\vct\eta}$. We generate again $24999$ surrogates and compute the null distribution of all elements of the empirical covariance matrix computed from $K=100$ observations. In Fig.~\ref{test-connectivity-directed} and Fig.~\ref{test-connectivity-undirected}, we show the detections for a statistical significance set at $\alpha=5\%$ respectively for directed- and undirected-graph surrogates, corrected for multiple comparisons. Most of the covariances  detected by directed-graph surrogates are TPs, while for the undirected-graph surrogates we observe less TPs and many more FPs, thus lowering the specificity of the test. Result of the detected pairs for permutation-based surrogates can be found in Supplementary Material.

The same testing scheme (same number of surrogates and $K=100$) is repeated on randomly selected empirical covariance matrices $\ma H_\eta$ of various density (from $1\%$ to $20\%$ of all possible connections). For each covariance matrix, an accuracy of detection can be computed from the maps of ground truth and detections. Finally aggregating the scores across repetitions, we show higher accuracies across all densities of the covariance matrices in directed model, than in undirected and permutation-based model Fig.~\ref{connectivity-distribution}. \added{Interestingly, the accuracy of the undirected surrogates is initially lower than that of the naive surrogates, but eventually surpasses it as the density increases. The null associated with the undirected model imposes graph-dependent covariance structures that does not match well with the simulated measurements that are perturbed DGWSS signals. Therefore, even at low density, the DGWSS structure of the empirical signal may induce false positives under a GWSS null, leading to lower accuracy. As the density increases, the undirected null becomes more appropriate, reducing false positive rate. By contrast, the permutation-based model assumes no graph dependence in its null, resulting in a smaller effect on the false positive rate across covariance densities (see also Supplementary Material).}

\subsection{Detection of Irregularities in Graph-Based Diffusion}
\label{test-diffusion}
Diffusion is commonly used to model signal evolution on a graph. The classical heat equation is the following differential equation
\begin{equation}
    \frac{\partial x(t)}{\partial t} = -\Delta x(t),
\end{equation}
where $\Delta$ is the Laplacian operator. The discretized heat equation on a graph is then given by
\begin{equation}
    \vc x_{t+1} = (\ma I - \ma L) \vc x_t,
\end{equation}
where $t$ is a time index and $\ma L$ the graph Laplacian. In directed graphs, we consider the random walk directed Laplacian $\ma L = \ma I - \ma D^{-1} \ma A$ with $\ma D$ the in-degree diagonal matrix of the adjacency matrix $\ma A$~\cite{singh_graph_2016}.

We now consider a signal generation model for a diffusion process $\vc m_{t}$:
\begin{equation}
    \vc x_{t} = \ma D^{-1}\ma A \vc x_{t-1} + \vct \epsilon_{t-1} + \vct \delta_{t-1},
\end{equation}
with $\ma A$ being the graph adjacency matrix, and contributions from a WN $(\vct \epsilon_t)_t$ as well as the graph signal $\vct \delta_t$ representing deterministic disruption that is to be detected. 

For the experimental setting, we re-use the USA graph from Fig.~\ref{dirac-signal} and set the length of the sequence to $T=20$ that is observed $K=50$ times. We thus obtain sequences $\vc m_{k,t}$ with $t=1,\ldots,T$ and $k=1,\ldots,K$, initialized by $\vc m_{k,0}$ that is a realization of WN on the directed graph, sampled from a multi-variate Gaussian distribution. From each sequence,  we generate $999$ directed, undirected and permutation based surrogates signals of the following form $\{{\vc m}_{k, t}^{(l)}, k=1,...,K, t=1,...,T\}_{l=1}^{999}$, and deploy them to build a null distribution of nodal signal as in Sect.~\ref{node-test}. Each node and timepoint is then tested. We show the feasibility of this scheme for different choices of $\vct \delta_t$.

We consider the regime where, given a set of nodes $\mathcal{N}$ and time index $t_0=9$, the disruption signal is defined as
\begin{equation}
\vct \delta_t[n]= \left\{ \begin{array}{ll} 
    a, & n\in \mathcal{N} \text{ and } t = t_0,\\
    0, & \text{otherwise}.
    \end{array} \right.  
\end{equation}
Practically, we set $a=10$ as in Sect.~\ref{test-node} and consider the WN $\vct \epsilon_t$ of the in-degree normalized USA graph for all $t=1,...,T$. 
We first show a generated graph signal observation in Fig.~\ref{ar-n1-test}, where $\mathcal{N}=\{n_0\}=\{1\}$ and $t_0=9$. We then vary the significance level $\alpha$ and measure specificity and sensitivity w.r.t.\ the ground truth, which allows to build the Receiver Operating Characteristic (ROC) curve for node $n_0$. Repeating the process for all the other nodes gives us a set of ROC curves (Fig.~\ref{ar-n1-roc}). We observe that surrogates built using the directed graph outperform both undirected and permutation-based surrogates in terms of sensitivity and false positive control.

With the same settings as when $\text{card}(\mathcal{N})=1$, we repeat over $50$ random sets $\mathcal{N}$ for a given size of $\mathcal{N}$ instead of repeating the test over all single nodes. We first show an observation of the timecourses in Fig.~\ref{ar-multi-test} and the associated ROC curves in Fig.~\ref{ar-n5-roc} both for $\text{card}(\mathcal{N})=5$. Then we also perform the same procedure but with increasing set sizes, such that $\text{card}(\mathcal{N})$ has node density varying from $5\%$ to $50\%$ and report the Area Under the Curve (AUC) of the ROC curves, as shown in Fig.~\ref{auc_multiple}
From this experiment of varying $\text{card}(\mathcal{N})$, we show clearly that regardless of the number of perturbations, the directed model allows for higher control on the false positives while remaining sensitive to significant points in time in comparison to both undirected and permutation-based models.

\begin{figure*}
    \centering
    \captionsetup[subfigure]{justification=centering}
  \subfloat[Detections using directed-graph surrogates\label{temperature-graph}]{%
  \includegraphics[width=0.29\linewidth]{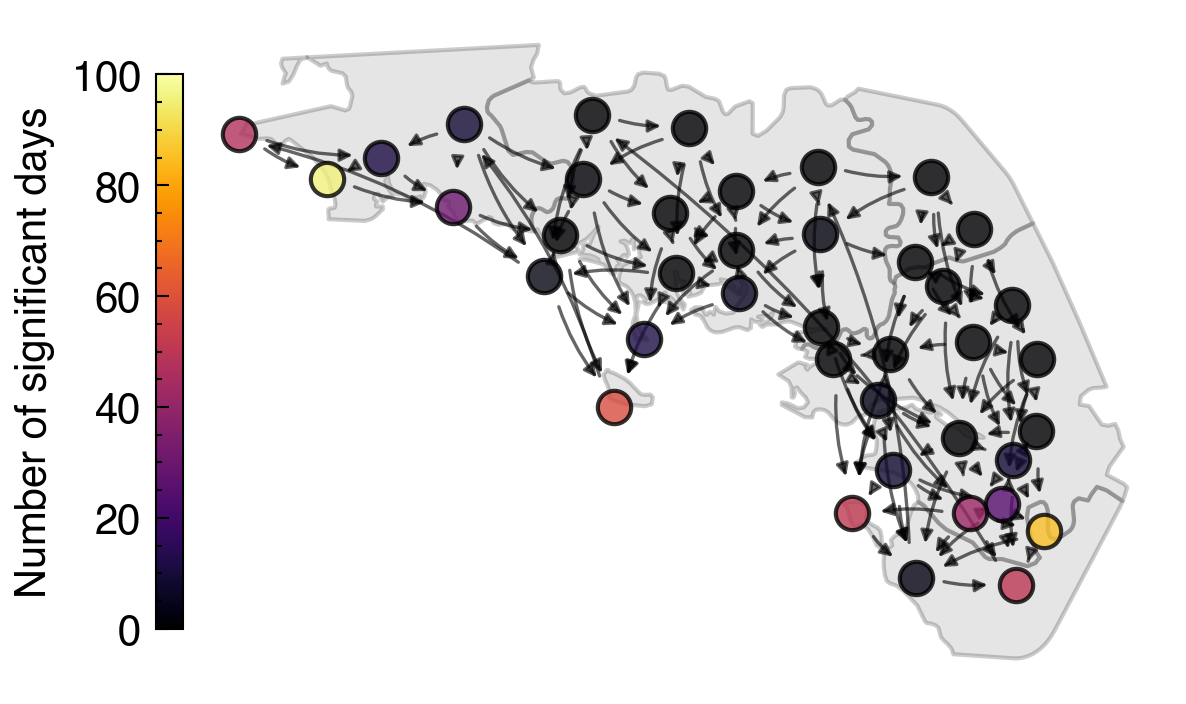}}
  \hfill
  \subfloat[Detections using undirected-graph surrogates\label{temperature-graph-undirect}]{%
  \includegraphics[width=0.25\linewidth]{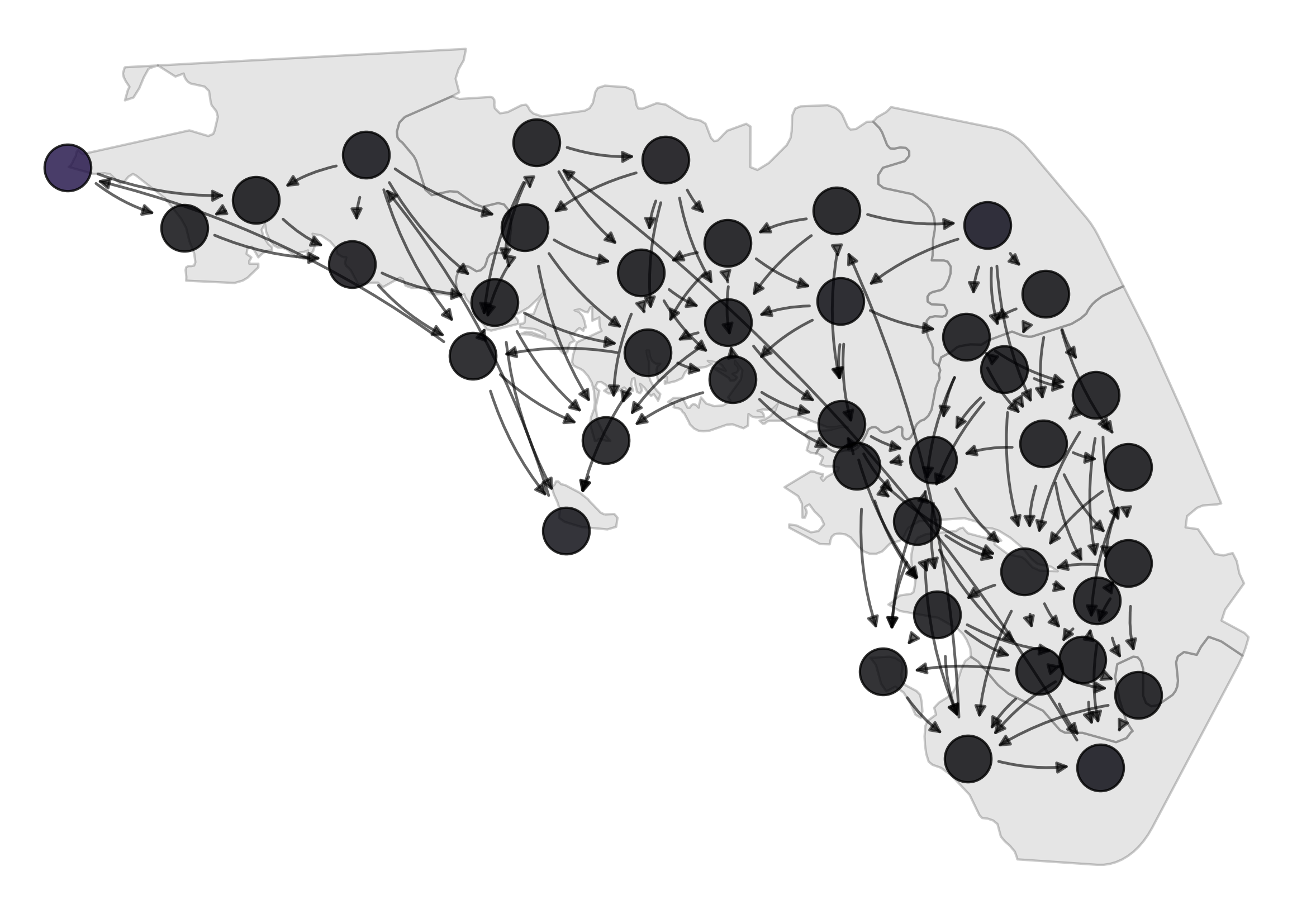}}
  \hfill
    \subfloat[Spectral covariance \label{autocorrelation}]{%
    \includegraphics[width=0.43\linewidth]{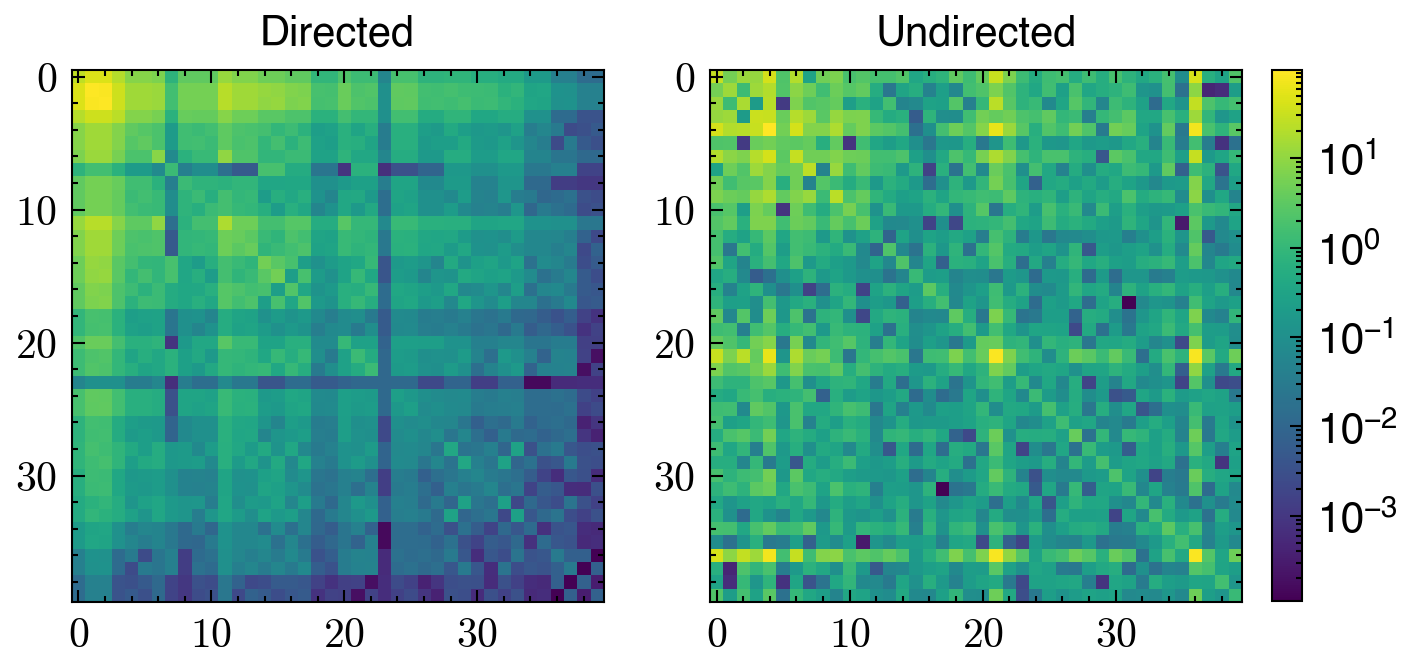}}
    
    \subfloat[Null distribution for inland station \label{noncoastal-station-surrogate}]{%
    \includegraphics[width=0.504\linewidth]{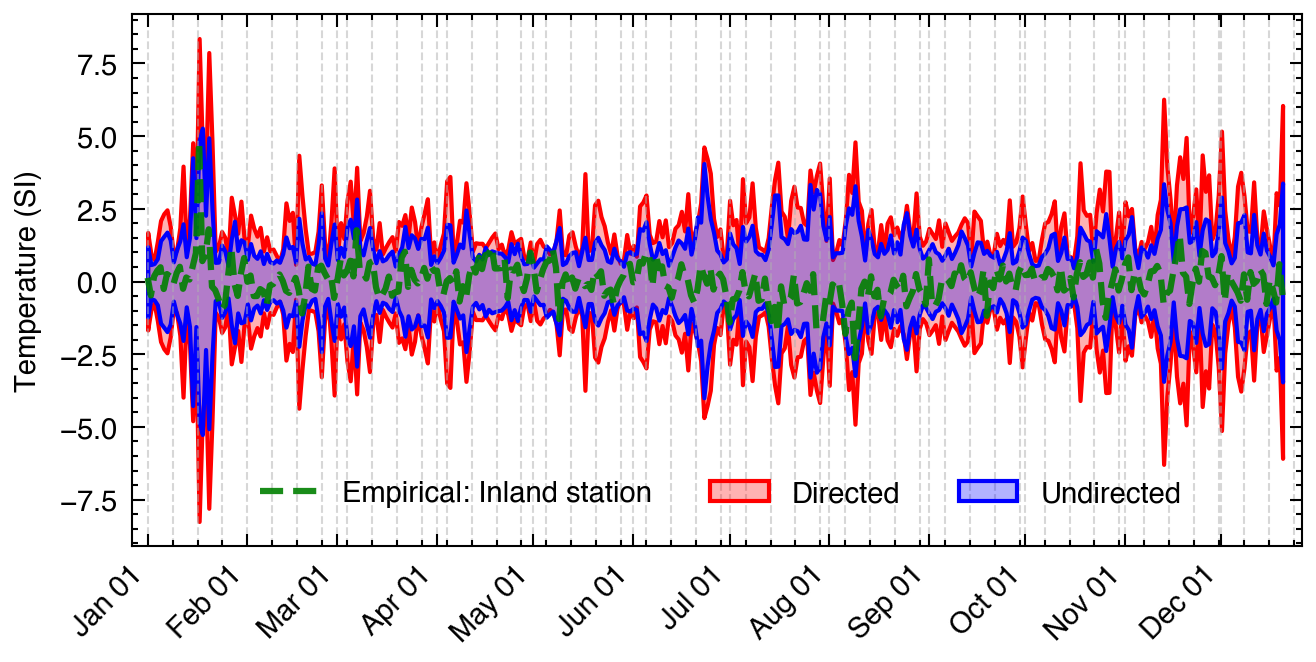}}
    \subfloat[Null distribution for coastal station \label{coastal-station-surrogate}]{%
    \includegraphics[width=0.496\linewidth]{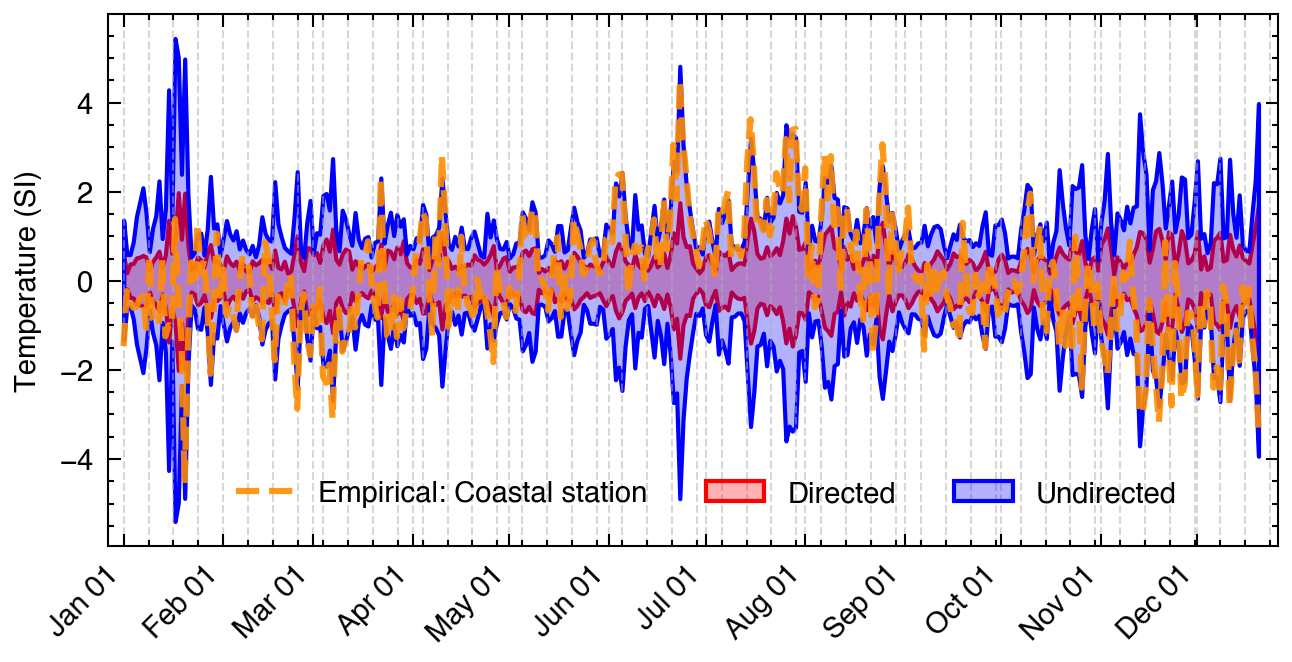}}
    
  \caption{(a)~Number of days with significant temperature detected by directed-graph surrogates. (b)~Number of days with significant temperature detected by undirected-graph surrogates. (c)~Spectral covariance matrix computed from covariance between stations. (d)~The measured temperatures (normalized) over the year for the inland station PLOERMEL (green), and its corresponding directed-graph (red) and undirected-graph surrogates (blue) at $[-\sigma, \sigma]$. (e)~The measured temperatures (normalized) over the year for the coastal station PLOVAN (orange), and its corresponding directed-graph (red) and undirected-graph surrogates (blue) at $[-\sigma, \sigma]$.}
\end{figure*}

\subsection{Application to Temperature Data on Sensor Network} \label{application} 
Finally, we apply the framework to real-world temperature data collected in 40 meteorological stations in the North-Western Coast of France over the full year of 2023 (Fig.~\ref{temperature-graph}), which yields $\vc m_k$, $k=1\,\ldots,354$, measurements after removing missing data. The graph is defined by associating nodes to stations. The edge weights derived from distance; i.e., for the edge $(n,m)$ the weight is $e^{-\lambda ||\vc p_n-\vc p_m||_2}$, where $\vc p_n$ represent the geographic coordinates (longitude and latitude) of node $n$. Additionally, we consider the $5$-nearest neighbours graph and choose the direction of the edge according to latitude; i.e., from lower to higher latitude as in \cite{marques_signal_2020, shafipour_directed_2019, iraji_wss_2025}. We ensure diagonalizability and invertibility of the GSO through small perturbation \cite{seifert_digraph_2021}, which is effectively adding 4 edges to a total of 149 edges. The undirected counterpart is obtained by symmetrizing the adjacency matrix $A$, yielding $\frac{1}{2}(A+A^T)$. The signal is further preprocessed by first centering over time, and then over nodes. This leads to a data matrix of size $40\times 354$. In Fig.~\ref{autocorrelation}, the spectral covariance matrix is shown for the full graph signals. The obtained stationary level is $\kappa(\ma H_\vc m)\approx 0.41$, $\bar{\kappa}(\ma H_\vc m)\approx 0.49$, for respectively directed and undirected case showing that the temperature timecourses are non-stationary with respect to the directed graph and its undirected counterpart.

We now seek to detect irregularity in the measurements that leads to the aforementioned non-stationarity considering a null hypothesis that assumes stationarity. We generate $9999$ surrogates per day to build a null distribution per node and per timepoint with initial significance level of $\alpha=5\%$ which is adjusted for multiple comparisons. In Figs.~\ref{temperature-graph} and \ref{temperature-graph-undirect}, the number of days for which the null is rejected are illustrated, obtained using surrogates on the directed and undirected graph, respectively. From undirected graph surrogates, a small number of irregularities are found (at most $19$ out of $354$ days). In contrast, for directed graph surrogates, we find that mainly coastal stations present irregularities (up to $90$) while almost none of the inland stations do. Such irregularities are present all over the year especially during summer and winter. Moreover, we can observe from Figs.~\ref{noncoastal-station-surrogate} and \ref{coastal-station-surrogate} that the null distributions presented by directed-graph surrogates are wider (larger variance) than undirected ones on non coastal stations while the converse is instead true for coastal stations. 

While we have no ground truth, a tentative explanation for these results can \added{be attributed to unmodeled}\deleted{originate from the orientation of the edges from lower-to-higher latitude, and the geography of the observed region, from land and higher altitudes in the north, to sea in the south.
Unmodeled} marine influences affecting coastal stations\replaced{. Owing to the high thermal inertia of the sea, coastal regions often exhibit temperatures that differ from other areas~\cite{scheitlin_maritime_2013}, which is likely}{are likely} responsible for the increased number of detections that do not obey the stationarity assumptions. This effect is more pronounced for the directed graph, whereas the undirected graph may partially mitigate marine influences due to the spatial similarity among surrogate signals at neighboring coastal nodes.

\subsection{Implementation}
Documented code is publicly available on a repository\footnote{\url{https://github.com/MIPLabCH/Graph-NonParam-Stats}} that contains an implementation in Python that can reproduce all algorithms and figures of this paper. 

\section{Discussion \& Outlook}
We revisited and leveraged the concept of stationarity on directed graphs. First, we show that white noise on directed graphs comes with a graph-specific, non-diagonal covariance matrix. Second, we propose a novel method for non-parametric hypothesis testing that exploits phase randomization of signals in the graph spectral domain. Finally, we demonstrate, using various examples with synthetic and real data, the benefits of considering surrogates on directed graphs compared to those for undirected graphs or naive permutation that do not fully exploit the graph structure. Here, we highlight several points for extension and discussion that could further increase the applicability and impact of the proposed framework.

Our starting point is that the covariance matrix is diagonalized by the eigendecomposition of the GSO, rather than only block-diagonal as in previous work~\cite{iraji_wss_2025}. Therefore, the covariance structure can be summarized by a length-$N$ PSD similar to the definition of GWSS. Moreover many real graphs are often strongly connected, enabling diagonalization of the GSO \cite{sevi_harmonic_2023}. In cases where this is not possible, recent graph adjustment algorithms~\cite{seifert_digraph_2021, stankovic_zero-padding_2023} may be applied, all in all to overcome the computational instability of the JNF. Most importantly, as demonstrated in recent work on the graph Hilbert transform~\cite{chan_hilbert_2025}, the JNF prevents the emergence of phase. For instance, the GSO for a directed path does not lead to conjugate eigenvectors nor does it provide phase information that can be exploited for ``shifting'' operations. \added{As a result, directional information encoded through phase shifts cannot be exploited by surrogate generation.}

These considerations relate to the distinction with graphical causal models that impose acyclicity of the graph \cite{pearl2009causality}. The rationale is that  directed acyclic graphs allow non-ambiguous identification of statistical dependencies. For our framework, such graphs would lead to a non-diagonalizable GSO requiring  adjustment as mentioned before. Specifically, the emergence of phase of the GFT coefficients, which is essential for the proposed randomization scheme, requires the presence of cycles in the graph~\cite{chan_hilbert_2025}. Therefore, both frameworks come with their own assumptions and applications.

The proposed approach also differs fundamentally from the surrogate algorithm introduced in \cite{belda_new_2019}. Although the authors employ a Hermitian Laplacian formulation, which to some extent can accommodate for directionality, they do not provide a constructive procedure to obtain a Hermitian Laplacian from a directed graph. The underlying construction relies on symmetric edge weights, with directionality encoded by the sign of the imaginary part which is selected by maximizing smoothness criterion on the observed signals. It is therefore, in general, not applicable to directed graphs. Importantly, the method in \cite{belda_new_2019} is applied on complex-valued observations for which the notion of graph stationarity was not defined. As a result, the corresponding surrogates are not grounded in stationarity considerations and do not leverage the notion of directed graph stationarity. Furthermore, the surrogate signals in \cite{belda_new_2019} are introduced for data augmentation purposes, rather than for statistical hypothesis testing, which is a central objective of the present work.

The framework has been devised for real-valued graph signals. However, since the definition of DGWSS extends to complex-valued random vectors, the methodology can be readily broadened to complex-valued data, both for measurements and surrogates. In that case, the test statistic needs to be adapted to deal with the complex nature of the signals; e.g., by considering separately the real and imaginary parts, or resort to circular statistics~\cite{jammalamadaka_topics_2001}. 

In addition to testing nodal values or covariance structures, various other test metrics have been defined in \cite{theiler_testing_1992, borgnat_testing_2010} to measure chaotic behaviours of non linear systems, such as Lyapunov exponents. These statistics could be adapted to the graph setting and employed within our framework to test for non-linear behavior on graphs. 

In this work, multiple-comparison effects are handled using standard correction procedures, such as the Bonferroni \added{and FDR} method, applied to the p-values obtained from the proposed surrogate-based nonparametric tests. In a similar way, recent advances in multiple hypothesis testing strategies on networks~\cite{jian2025graph} could be used, exploiting the underlying graph structure to control for false discovery rate while improving detection power.

Additionally, one can observe that in our method, no explicit assumption with regards to the probability distribution of the random signal was made. However, similarly to phase-randomized surrogate generation for temporal signals~\cite{theiler_testing_1992}, the present approach implicitly tests observations against Gaussian distributed surrogates. This arises from the fact that phase randomization only preserves autocovariance, such that, when considering a large number of surrogates, the resulting phase-randomized surrogates are asymptotically Gaussian due to the central limit theorem \cite{Maiwald2008}. As originally proposed in \cite{theiler_testing_1992}, an amplitude-adjusted surrogate generation scheme can be employed to extend the null hypothesis to non-Gaussian processes.

\added{Throughout this work, we employ the adjacency matrix as the GSO, following a previous study that introduced stationary process for directed graphs using the adjacency operator \cite{iraji_wss_2025}. Nevertheless, the directed graph Laplacian operator could also be employed to characterize stationarity, in a manner analogous to the GWSS for undirected graph \cite{perraudin_stationary_2017}. In both cases, the associated spectral representation relies on the eigen- or Jordan decomposition. By contrast, alternative decompositions based on singular-values \cite{wei_vertex-frequency_2025} or optimized orthogonal bases \cite{shafipour_directed_2019} produce real-valued singular values and components, which hinders an immediate extension of the uniform phase randomization mechanism central to our framework. }\replaced{This observation would instead motivate the investigation of}{Future work could study the framework by considering} alternative matrix factorizations \added{that possess spectral conjugate structures}, such as the Schur decomposition \cite{xiao_joint_2023} which comprise 2-by-2 matrix of complex conjugate eigenvalues pairs, or the polar decomposition \cite{kwak_frequency_2024},  which exploits a real-valued unitary factor whose eigenvalues lie on the unit circle and come in conjugate pairs. These factorizations may enable other types of randomizations that also leverage conjugate pairs. 


%

\section*{Supplementary Material}
Additional plots are listed in Supplementary Material.




\ifCLASSOPTIONcaptionsoff
  \newpage
\fi



\bibliographystyle{ieeetr}
%





\bibliography{references}

\end{document}